\documentclass[english,10pt, a4paper, final]{article}
\PassOptionsToPackage{obeyFinal}{todonotes}
\usepackage[T1]{fontenc}
\usepackage[latin9]{inputenc}
\usepackage{geometry}
\usepackage{color}
\usepackage{babel}
\usepackage{units}
\usepackage{mathtools}
\usepackage{bm}
\usepackage{todonotes}
\usepackage{amsmath}
\usepackage{amsthm}
\usepackage{amssymb}
\usepackage[numbers]{natbib}
\usepackage{microtype}
\usepackage[unicode=true,
bookmarks=true,bookmarksnumbered=false,bookmarksopen=false,
breaklinks=false,pdfborder={0 0 0},pdfborderstyle={},backref=page,colorlinks=true]
{hyperref}
\hypersetup{pdftitle={Kernels},
	pdfauthor={Paul Dommel},
	citecolor=darkgray, urlcolor= darkgray, linkcolor= darkgray}

\makeatletter
\numberwithin{equation}{section}
\theoremstyle{plain}
\newtheorem{thm}{\protect\theoremname}[section]
\theoremstyle{remark}
\newtheorem{rem}[thm]{\protect\remarkname}
\theoremstyle{plain}
\newtheorem{prop}[thm]{\protect\propositionname}
\theoremstyle{plain}
\newtheorem{lem}[thm]{\protect\lemmaname}

\usepackage[notcite,notref]{showkeys}
\usepackage{dsfont}   
\usepackage{newtxtext, newtxmath}	
\usepackage{microtype}
\usepackage{tikz}
\usetikzlibrary{matrix,arrows.meta}
\usepackage{algorithm,algpseudocode}

\usepackage{doi}
\usepackage{orcidlink}

\DeclareMathOperator{\E}{{\mathds E}}

\DeclareMathOperator*{\Tr}{trace}

\newcommand{\one}{{\mathds 1}}

\makeatother

\providecommand{\lemmaname}{Lemma}
\providecommand{\propositionname}{Proposition}
\providecommand{\remarkname}{Remark}
\providecommand{\theoremname}{Theorem}

\begin{document}
\title{On the Approximation of Kernel functions}
	\author{Paul Dommel%
	\thanks{University of Technology, Chemnitz, Faculty of mathematics. 90126 Chemnitz, Germany\protect }  
	\and  Alois Pichler\footnotemark[1]\,%
	\thanks{\orcidlink{0000-0001-8876-2429}~\protect\href{https://orcid.org/0000-0001-8876-2429}{orcid.org/0000-0001-8876-2429}.
		Contact: \protect\href{mailto:alois.pichler@math.tu-chemnitz.de}{alois.pichler@math.tu-chemnitz.de}}}
\maketitle
\begin{abstract}
Various methods in statistical learning build on kernels considered
in reproducing kernel Hilbert spaces. In applications, the kernel
is often selected based on characteristics of the problem and the
data. This kernel is then employed to infer response variables at
points, where no explanatory data were observed. 

The data considered here are located in compact sets in higher dimensions
and the paper addresses approximations of the kernel itself. The new
approach considers Taylor series approximations of radial kernel functions.
For the Gauss kernel on the unit cube, the paper establishes an upper
bound of the associated eigenfunctions, which grows only polynomially
with respect to the index. The novel approach substantiates smaller
regularization parameters than considered in the literature, overall
leading to better approximations. This improvement confirms low rank
approximation methods such as the Nyström method. 

\medskip

\noindent \textbf{Keywords:} statistical learning $\cdot$ kernel methods $\cdot$  reproducing kernel Hilbert spaces $\cdot$  Nyström method
\end{abstract}

\section{Introduction}

This paper contributes to statistical methods building on reproducing
kernel Hilbert spaces. These methods have become popular in statistical
learning, in inference and in support vector machines due to the kernel
trick. They constitute powerful tools in different scientific areas
such as geostatistics (cf.\ \citealt{Honarkhah2010}), stochastic optimization
(cf.\ \citealt{DommelPichler2023}, \citealt{Hanasusanto}), digit recognition
(cf.\ \citealt{Scholkopf1997}), computer vision (cf.\ \citealt{ZhangComputerVision})
and bio informatics (cf.\ \citealt{Schlkopf2005Biologie}). 

The approach presented here approximates the kernel function by elements
in the range of the associated Hilbert\textendash Schmidt integral
operator. We choose these elements so that its Taylor series expansion
matches the initial coefficients. The method applies for general radial
kernels with variable bandwidth. Special emphasize is given to the
Gaussian kernel, which is of major importance in practical applications. 

Fundamental in statistical approximation is the regularization parameter.
Standard results suggest regularization parameters decreasing as $\mathcal{O}(1/n)$,
where $n$ is the sample size. This paper, in contrast, justifies
significantly smaller regularization parameters, often close to machine
precision. This leads to enhanced approximation quality.

An additional consequence of our approach, not sufficiently addressed
in the literature yet, are the magnitudes of the eigenfunctions of
the related Hilbert\textendash Schmidt operator. We demonstrate that
this magnitude grows only polynomially in the index. Exploiting this
crucial property, we present an interpolation inequality, which
allows bounding the uniform error by the much weaker $L^{2}$\nobreakdash-error. 

The approximation of the kernel function builds on the function $w_{m}^{x}(\cdot)$,
smallest in $L^{2}([0,1])$\nobreakdash-norm, satisfying the moment
constraints 

\begin{equation}
\int_{0}^{1}z^{\ell}w_{m}^{x}(z)\,dz=x^{\ell},\qquad\ell=0,\dots,m-1.\label{eq:MomProb}
\end{equation}
We provide the explicit solution, which is a polynomial with coefficients
involving the Hilbert matrix. 

The results have consequences for low rank kernel methods. For these
methods, they ensure a stable approximations quality by building only
on few supporting points. Prominent examples of these methods include
the Nyström algorithms and kernel principal component analysis. 

\paragraph{Related literature. }

Our results address the approximation of randomly located kernel functions.
This is of particular importance for low rank kernel methods, which
build on an approximation of the kernel matrix itself. The Nyström
method, introduced in \citet{WlliamsSeeger2000}, is a prominent example
for this technique. \citet{Mahoney2005} analyze the error of the
matrix approximation in the Nyström method and relate it to the best
approximating matrix, while \citet{Bach2013} studies the precision
of predictions directly. The excellent work of \citet{NIPS2015_03e0704b}
relates the rank of the approximating matrix to the approximation
of kernel functions. The result is then employed to develop a low
rank regression approach, which is significantly cheaper in computations
than kernel ridge regression while maintaining stable prediction accuracy,
cf.\ \citet{Rudi2017FALKONAO}. Kernel principal component analysis builds on these results as well, cf.\ \citet{Sriperumbudur2021}.

The approximation of kernel functions is the core research question
of this paper, for which we present new bounds. The second main result
of this work addresses the eigensystem of Mercer's decomposition associated
with the Gaussian kernel. \citet{Belkin2009} address this issue for
an unbounded domain, building on the normal distribution as underlying
design measure. The authors provide an explicit expression of the
eigenvalues and eigenfunctions by involving the Hermite polynomials
in an unbounded domain. In compact domains, \citet{Diaconis2008}
consider the eigensystem for the Laplacian kernel.

The analysis of the divide and conquer approach relies on properties
of the eigenfunctions as well, cf.\ \citet{Zhang2013}, but the paper
builds on unverified assumptions. The analysis of the regression error
in different norms (cf.\ \citealt{steinwartfischer2020} and \citealt{Steinwart2009})
can be related to bounded eigenfunctions as well. This paper presents
explicit bounds of the maximum value the eigenfunction may attain.

\paragraph{Outline of the paper. }

Section~\ref{sec:MinimalMoment} addresses polynomial approximations
of kernel functions. Section~\ref{sec:Gauss} addresses the Gaussian
kernel specifically and presents bounds of the eigenfunctions on bounded
domains. This section contains the main results, which are considered
in Section~\ref{sec:Implications} in applications. Section~\ref{sec:Conclusion}
concludes the paper. 

\section{The minimal moment function\label{sec:MinimalMoment}}

Reproducing kernel Hilbert spaces (RKHS) build on a kernel function,
denoted~$k$. The approximations considered here build on the point
evaluation function $k_{x}(\cdot)\coloneqq k(x,\cdot)$. This section
resumes the minimal moment function $w_{m}^{x}$ from~\eqref{eq:MomProb},
for which the image $L_{k}w_{m}^{x}$ is a suitable estimate of $k_{x}$.
We first derive the function~$w_{m}^{x}$ for the simple design space
$\mathcal{X}=[0,1]$, and then extend it to the $d$-dimensional case
with some non-uniform design measure $P$. Using the moment property~\eqref{eq:MomProb},
we derive an error bound for the approximation quality of $k_{x}(\cdot)$
by $L_{k}w_{m}^{x}$, which is based on the Taylor coefficients of the
kernel.

\subsection{\label{subsec:3.1} The minimal moment function}

The central element of this paper is the function with smallest $L^{2}$\nobreakdash-norm,
satisfying the moment properties~\eqref{eq:MomProb}. There are infinitely
many functions fulfilling the condition~\eqref{eq:MomProb}. We refer
to the function with smallest $L^{2}$\nobreakdash-norm as the \emph{minimal
moment function}, where the inner product is $\left\langle f,g\right\rangle _{L^{2}}\coloneqq\int_{\mathcal{X}}f(z)g(z)p(z)\,dz$ with the density $p$ of the underlying measure $P$.
Throughout this section, we consider the design space $\mathcal{X}=[0,1]$
equipped with the uniform measure $P\sim\mathcal{U}[0,1]$.

In what follows we provide an explicit representation of the minimal
moment function. We demonstrate that it is a polynomial of degree
$m-1$, with coefficients originating from a Hilbert matrix. 
\begin{thm}[Explicit minimal moment function]
\label{thm:OptProb} For $x\in[0,1]$ fixed, the optimization problem
\begin{equation}
\vartheta^{*}\coloneqq\min\left\{ \left\Vert w\right\Vert _{L^{2}}^{2}\colon\int_{0}^{1}z^{\ell}w(z)dz=x^{\ell},\ \ell=0,1,\dots,m-1\right\} \label{eq:OptConstr}
\end{equation}
has the unique solution 
\begin{equation}
w_{m}^{x}(z)=\sum_{i=1}^{m}\alpha_{x,i}z^{i-1},\qquad z\in(0,1),\label{eq:DefMomFunc}
\end{equation}
where $\alpha_{x}$ satisfies the equations $H_{m}\alpha_{x}=\bar{x}$
for the Hilbert matrix $H_{m}\coloneqq\left(\frac{1}{i+j-1}\right)_{i=1,j=1}^{n,n}$
and the vector $\bar{x}\coloneqq(1,x,\dots,x^{m-1})\in\mathbb{R}^{m}$.
\end{thm}

\begin{proof}
The Lagrangian $\mathcal{L}$ of~\eqref{eq:OptConstr} is 
\[
\mathcal{L}(w,\mu)=\left\langle w(z),w(z)\right\rangle _{L^{2}}+\sum_{i=1}^{m}\mu_{i}\left(\left\langle z^{i-1},w(z)\right\rangle _{L^{2}}-x^{i-1}\right),
\]
where $w\in L^{2}$ and $\mu=(\mu_{1},\dots,\mu_{m})\in\mathbb{R}^{m}$
are Lagrange multipliers. The first order condition reads 
\begin{align*}
(\nabla_{w}\mathcal{L})(w_{\mu}^{\ast},\mu) & =2w_{\mu}^{\ast}+\sum_{i=1}^{m}\mu_{i}z^{i-1}=0,
\end{align*}
which is equivalent to $w_{\mu}^{\ast}(z)=-\frac{1}{2}\sum_{i=1}^{m}\mu_{i}z^{i-1}$.
For $\mu=(\mu_{1},\dots,\mu_{m})$ fixed, the Lagrangian function
$\mathcal{L}$ is convex and~$w_{\mu}^{\ast}$ thus a minimizer.
Hence, the Lagrangian dual function is
\begin{align*}
g(\mu) & \coloneqq\min_{w\in L^{2}}\mathcal{L}(w,\mu)=\mathcal{L}(w_{\mu}^{\ast},\mu)=\left\langle w_{\mu}^{\ast},w_{\mu}^{\ast}\right\rangle _{L^{2}}+\sum_{i=1}^{m}\mu_{i}\left(\left\langle z^{i-1},w_{\mu}^{\ast}\right\rangle _{L^{2}}-x^{i-1}\right),
\end{align*}
depending only on the multipliers $\mu$. As $\left\langle z^{i-1},z^{j-1}\right\rangle _{L^{2}}=\frac{1}{i+j-1}$,
we further have that 
\[
\left\langle z^{i-1},w_{\mu}^{\ast}\right\rangle _{L^{2}}-x^{i-1}=-\frac{1}{2}\sum_{i=1}^{m}\mu_{i}^{\ast}\frac{1}{j+i-1}-x^{j-1}=0
\]
by setting $\mu^{\ast}=-2H_{m}^{-1}\bar{x}$. It follows that 
\[
g(\mu^{\ast})=\left\langle w_{\mu^{\ast}}^{\ast},w_{\mu^{\ast}}^{\ast}\right\rangle _{L^{2}}\ge\vartheta^{\ast},
\]
as $w_{\mu^{\ast}}^{\ast}$ is feasible in~\eqref{eq:OptConstr}.
This implies strong duality as well as the optimality of $w_{\mu^{\ast}}^{\ast}=w_{m}^{x}$,
which is the assertion.
\end{proof}
A minimal moment function of particular interest is the optimizer
of~\eqref{eq:OptConstr} associated with the point $x=1$. In contrast
to the general case $x\in[0,1]$, the $L^{2}$\nobreakdash-norm of
$w_{m}^{1}$ can be computed explicitly. Indeed, it holds that 
\begin{align}
\int_{0}^{1}w_{m}^{1}(z)^{2}dz & =\int_{0}^{1}\sum_{i=1}^{m}\sum_{j=1}^{m}\alpha_{1,i}\alpha_{1,j}z^{i+j-2}dz\nonumber \\
 & =\sum_{i=1}^{m}\sum_{j=1}^{m}\alpha_{1,i}\alpha_{1,j}\frac{1}{i+j-1}\nonumber \\
 & =\sum_{i=1}^{m}\alpha_{1,i}1=\alpha_{1}^{\top}e=e^{\top}H_{m}^{-1}e=m^{2},\label{eq:7}
\end{align}
as $\sum_{i,j=1}^{n}(H_{m}^{-1})_{i,j}=m^{2}$. 

In what follows we bound the norm of the remaining moment functions.
For that we relate $w_{m}^{x}$ with a specific linear transformation
of $w_{m}^{1}$.
\begin{thm}[Upper bound of the weight function]
The weight function $w_{m}^{x}$ satisfies the bound
\begin{equation}
\left\Vert w_{m}^{x}\right\Vert _{L^{2}}^{2}\le m^{2}\label{eq:NormEst}
\end{equation}
for any $x\in[0,1]$.
\end{thm}

\begin{proof}
Note first that, for $x=0$, 
\[
\int_{0}^{1}w_{m}^{0}(z)^{2}dz=e_{1}^{\top}H_{m}^{-1}e_{1}=\bigl(H_{m}^{-1}\bigr)_{1,1}=m^{2},
\]
where $e_{1}=(1,0,\dots.0)$ is the first vector in the canonical
basis. To verify the bound for the remaining points $x\in(0,1)$ we
relate $w_{m}^{x}$ with an auxiliary function $\tilde{w}_{m}^{x}$
for which we are able to calculate the norm more easily. Setting $g^{x}(z)\coloneqq w_{m}^{1}(\frac{z}{x})\one_{[0,x]}(z)$
we define the auxiliary function
\[
\tilde{w}_{m}^{x}(z)\coloneqq\begin{cases}
g^{x}(z) & \text{if }z\le x,\\
g^{1-x}(1-z) & \text{if }z>x,
\end{cases}
\]
with $x\in(0,1).$ To relate $w_{m}^{x}$ with $\tilde{w}_{m}^{x}$,
we demonstrate first that $\tilde{w}_{m}^{x}$ satisfies the moment
constraints~\eqref{eq:MomProb} and bound its norm afterwards. It
holds that
\begin{equation}
\int_{0}^{1}z^{h}g^{x}(z)dz=\int_{0}^{x}z^{h}\,w_{m}^{1}\Bigl(\frac{z}{x}\Bigr)\,dz=x\int_{0}^{1}(yx)^{h}w_{m}^{1}(y)dy=xx^{h}\label{eq:1-1}
\end{equation}
for the first part of the integral. We further have that
\begin{align*}
\int_{0}^{1}z^{h}g^{1-x}(1-z)dz & =-\int_{1-x}^{0}(1-y)^{h}g^{1-x}(y)dz=\int_{0}^{1-x}(1-y)^{h}g^{1-x}(y)dy
\end{align*}
after changing the variables. By the binomial theorem, 
\begin{align*}
\int_{0}^{1-x}(1-y)^{h}g^{1-x}(y)dy & =\sum_{p=0}^{h}\binom{h}{p}(-1)^{h-p}\int_{0}^{1-x}y^{h-p}g^{1-x}(y)dy\\
 & =\sum_{p=0}^{h}\binom{h}{p}(-1)^{h-p}(1-x)(1-x)^{h-p}\\
 & =(1-x)(1-(1-x))^{h}=(1-x)x^{h},
\end{align*}
as $\int_{0}^{1-x}y^{h-p}g^{1-x}(y)dy=(1-x)(1-x)^{h-p}$ by~\eqref{eq:1-1}.
Connecting both identities, we have that
\[
\int_{0}^{1}z^{h}\tilde{w}_{m}^{x}(z)dz=\int_{0}^{x}z^{h}g^{x}(z)dz+\int_{0}^{1-x}(1-y)^{h}g^{1-x}(y)dy=x^{h},
\]
and thus the moment property~\eqref{eq:MomProb} of $\tilde{w}_{m}^{x}$.
It is now evident by~\eqref{eq:OptConstr} that $\left\Vert \tilde{w}_{m}^{x}\right\Vert _{L^{2}}$
is an upper bound of $\left\Vert w_{m}^{x}\right\Vert _{L^{2}}.$
Employing the same substitutions as above, we get that

\begin{align*}
\int_{0}^{1}\tilde{w}_{m}^{x}(z)^{2}dz & =\int_{0}^{x}g^{x}(z)^{2}dz+\int_{x}^{1}g^{1-x}(1-z)^{2}dz\\
 & =x\int_{0}^{1}w_{m}^{1}(y)^{2}dy+(1-x)\int_{\text{0}}^{1}w_{m}^{1}(y)^{2}dz\\
 & =\int_{0}^{1}w_{m}^{1}(y)^{2}dy=m^{2}
\end{align*}
by~\eqref{eq:7}, concluding the proof.
\end{proof}
The squared norm of the moments functions at the boundary points is
$m^{2}.$ However, the norm of the minimal moment functions associated
with the interior points $x\in(0,1)$ might be significantly smaller.

\subsection{Extensions and error estimates}

The results of the preceding Section~\ref{subsec:3.1} crucially
rely on the proposed setting, i.e., the design space $[0,1]$ equipped
with the uniform distribution. These assumptions are quite restrictive
and need to be relaxed for situations of practical application. To
this end we investigate a more general setting throughout this section.

We consider the multivariate case where $\mathcal{X}=[0,1]^{d}$,
with some underlying design measure $P$. This measure has a strictly
positive density with \[\infty > C > \sup_{z\in[0,1]^{d}}p(z)\ge\inf_{z\in[0,1]^{d}}p(z)>c>0,\] giving rise
to the inner product 
\[
\left\langle f,g\right\rangle _{L^{2}}=\int_{[0,1]^{d}}f(z)g(z)p(z)dz.
\]

In what follows we specify the minimal moment functions for this more
general setting. Building on the univariate moment property of the
functions~\eqref{eq:DefMomFunc}, we demonstrate that their product
satisfies a multivariate version of~\eqref{eq:MomProb}. The next
proposition reveals the precise statement. 
\begin{prop}[Upper bound of the weight function in higher dimensions]
\label{prop:MomDim} Let $x=(x_{1},\dots,x_{d})\in[0,1]^{d}$ and
consider the function 

\begin{equation}
W_{m}^{x}(z_{1},\dots,z_{d})\coloneqq\left(\prod_{i=1}^{d}w_{m}^{x_{i}}(z_{i})\right)p(z_{1},\dots,z_{n})^{-1},\label{eq:PointApprox}
\end{equation}
where $w_{m}^{x_{i}}$ are the functions defined in~\eqref{eq:DefMomFunc}.
The function $W_{m}^{x}$ satisfies the general moment property 
\begin{equation}
\int_{[0,1]^{d}}\left(\left\Vert y-z\right\Vert _{2}^{2}\right)^{\ell}W_{m}^{x}(z)p(z)dz=\left(\left\Vert y-x\right\Vert _{2}^{2}\right)^{\ell}\label{eq:MomProbGen}
\end{equation}
for all integers $\ell\le\frac{m}{2}$. Its norm is bounded by $\left\Vert W_{m}^{x}\right\Vert _{L^{2}}^{2}\le c_{p}m^{2d}$
with $c_{p}=\sup_{z\in[0,1]^{d}}p(z)^{-1}.$ 
\end{prop}


\begin{proof}
The moment property \eqref{eq:MomProbGen} follows from
\begin{align*}
\int_{[0,1]^{d}}\left(\left\Vert y-z\right\Vert _{2}^{2}\right)^{\ell}W_{m}^{x}(z)p(z)dz & =\int_{0}^{1}\dots\int_{0}^{1}\left(\sum_{i=1}^{d}(y_{i}-z_{i})^{2}\right)^{\ell}\prod_{i=1}^{d}w_{m}^{x_{i}}(z_{i})dz_{1}\dots dz_{d}\\
 & =\sum_{h_{1}+\dots+h_{d}=\ell}\binom{\ell}{h_{1},\dots,h_{d}}\prod_{i=1}^{d}\int_{0}^{1}(y_{i}-z_{i})^{2h_{i}}w_{m}^{x_{i}}(z_{i})dz_{i}\\
 & =\sum_{h_{1}+\dots+h_{d}=\ell}\binom{\ell}{h_{1},\dots,h_{d}}\prod_{i=1}^{d}(y_{i}-x_{i})^{2h_{i}}\\
 & =\left(\left\Vert y-x\right\Vert _{2}^{2}\right)^{\ell},
\end{align*}
as $\int_{0}^{1}z_{i}^{\ell_{i}}w_{m}^{x_{i}}(z_{i})=x_{i}^{\ell_{i}}$
holds for all integers $\ell_{i}\le m-1$. This is the first assertion. 

For the second assertion note that $\left\Vert w_{m}^{x_{i}}\right\Vert _{L^{2}}^{2}\le m^{2}$
holds by \eqref{eq:NormEst}. Hence, we get that
\begin{align*}
\int_{[0,1]^{d}}(W_{m}^{x}(z))^{2}p(z)dz & =\int_{0}^{1}\dots\int_{0}^{1}\left(\prod_{i=1}^{n}w_{m}^{x_{i}}(z_{i})\right)^{2}p(z_{1},\dots,z_{n})^{-1}dz_{1}\dots dz_{d}\\
 & \le\sup_{z\in[0,1]^{d}}\left|p(z)^{-1}\right|\prod_{i=1}^{d}\int_{0}^{1}(w_{m}^{x_{i}}(z_{i}))^{2}dz_{i}=c_{p}m^{2d},
\end{align*}
which concludes the proof.
\end{proof}
\begin{rem}
The function $W_{m}^{x}(\cdot)$ might not be the norm minimal function
satisfying~\eqref{eq:MomProbGen}. Indeed, the product structure
of~\eqref{eq:PointApprox} leads to an exponentially (with respect
to the dimension) increasing norm of $W_{m}^{x}$. However, the construction
of the norm minimal moment function satisfying~\eqref{eq:MomProbGen}
might require a significantly more involved representation, which
is out of the scope of this paper.
\end{rem}

Continuing the general setting considered above, we now provide the
first error estimate. To this end we consider a radial kernel 
\begin{equation}
k(x,y)=\phi(\left\Vert x-y\right\Vert ^{2})\label{eq:nonGauss}
\end{equation}
 as well as the corresponding integral operator $L_{k}\colon L^{2}(\mathcal{X},P)\to L^{2}(\mathcal{X},P)$ defined as 
\begin{equation}
(L_{k}f)(y)=\int_{\mathcal{X}}k(z,y)f(z)p(z)dz.\label{eq:Mercer}
\end{equation}
Here, $\phi$ is a smooth function with Taylor series expansion
\begin{equation}
\phi(x)=\sum_{\ell=0}^{\infty}\frac{\alpha_{\ell}}{\ell!}x^{\ell}.\label{eq:TaylorSer}
\end{equation}
Building on the moment property~\eqref{eq:MomProbGen}, we utilize
that the $\ell$th moment of $k_{x}-L_{k}W_{m}^{x}$ vanishes. The
subsequent theorem reveals the precise bound. 
\begin{thm}[Uniform bound in $d$\nobreakdash-dimensions]
\label{thm:TaylorBound} With the function $W_{m}^{x}$ defined in~\eqref{eq:PointApprox},
the error estimate 
\begin{equation}
\sup_{x\in[0,1]^{d}}\left\Vert (L_{k}W_{m}^{x})(y)-k_{x}\right\Vert _{\infty}\le(1+c_{p}^{\nicefrac{1}{2}}m^{d})\sum_{\ell=\left\lfloor \frac{m-1}{2}\right\rfloor +1}^{\infty}\frac{\left|a_{\ell}\right|}{\ell!}d^{\ell},\label{eq:TaylorErrorMoreDim}
\end{equation}
holds true. Here, $\left\lfloor \cdot\right\rfloor $ denotes the
floor function.
\end{thm}

\begin{proof}
Employing the series representation~\eqref{eq:TaylorSer} of $\phi$,
we have the decomposition
\begin{align*}
(L_{k}W_{m}^{x})(y) & =\int_{\mathcal{X}}\phi\left(\left\Vert y-z\right\Vert _{2}^{2}\right)W_{m}^{x}(z)p(z)dz\\
 & =\int_{\mathcal{X}}\sum_{\ell=0}^{\left\lfloor \frac{m-1}{2}\right\rfloor }\frac{a_{\ell}}{\ell!}(\left\Vert y-z\right\Vert _{2}^{2})^{\ell}W_{m}^{x}(z)p(z)dz\\
 & +\int_{\mathcal{X}}\sum_{\ell=\left\lfloor \frac{m-1}{2}\right\rfloor +1}^{\infty}\frac{a_{\ell}}{\ell!}(\left\Vert y-z\right\Vert _{2}^{2})^{\ell}W_{m}^{x}(z)p(z)dz.
\end{align*}
For the first part, it follows from the moment property~\eqref{eq:MomProbGen}
of $W_{m}^{x}$ that
\begin{align*}
\int_{\mathcal{X}}\sum_{\ell=0}^{\left\lfloor \frac{m-1}{2}\right\rfloor }\frac{a_{\ell}}{\ell!}(\left\Vert y-z\right\Vert _{2}^{2})^{\ell}W_{m}^{x}(z)p(z)dz & =\sum_{\ell=0}^{\left\lfloor \frac{m-1}{2}\right\rfloor }\frac{a_{\ell}}{\ell!}\left(\left\Vert y-x\right\Vert _{2}^{2}\right)^{\ell},
\end{align*}
 as the $2\left\lfloor \frac{m-1}{2}\right\rfloor \le m-1$. Hence,
we have that
\begin{align*}
\MoveEqLeft[3]\left|k(y,x)-(L_{k}W_{m}^{x})(y)\right|\\
= & \left|k(y,x)-\sum_{\ell=0}^{\left\lfloor \frac{m-1}{2}\right\rfloor }\frac{a_{\ell}}{\ell!}(\left\Vert y-x\right\Vert _{2}^{2})^{\ell}+\int_{\mathcal{X}}\sum_{\ell=\left\lfloor \frac{m-1}{2}\right\rfloor +1}^{\infty}\frac{a_{\ell}}{\ell!}(\left\Vert y-z\right\Vert _{2}^{2})^{\ell}W_{m}^{x}(z)p(z)dz\right|\\
= & \left|\sum_{\ell=\left\lfloor \frac{m-1}{2}\right\rfloor +1}^{\infty}\frac{a_{\ell}}{\ell!}(\left\Vert y-x\right\Vert _{2}^{2})^{\ell}+\int_{\mathcal{X}}\sum_{\ell=\left\lfloor \frac{m-1}{2}\right\rfloor +1}^{\infty}\frac{a_{\ell}}{\ell!}(\left\Vert y-z\right\Vert _{2}^{2})^{2\ell}W_{m}^{x}(z)p(z)dz\right|\\
\le & \sum_{\ell=\left\lfloor \frac{m-1}{2}\right\rfloor +1}^{\infty}\frac{\left|a_{\ell}\right|}{\ell!}d^{\ell}+\sum_{\ell=\left\lfloor \frac{m-1}{2}\right\rfloor +1}^{\infty}\frac{\left|a_{\ell}\right|}{\ell!}d^{\ell}\left\Vert W_{m}^{x}\right\Vert _{L^{2}}\\
\le & \sum_{\ell=\left\lfloor \frac{m-1}{2}\right\rfloor +1}^{\infty}\frac{\left|a_{\ell}\right|}{\ell!}d^{\ell}+c_{p}^{\nicefrac{1}{2}}m^{d}\sum_{\ell=\left\lfloor \frac{m-1}{2}\right\rfloor +1}^{\infty}\frac{\left|a_{\ell}\right|}{\ell!}d^{\ell}.
\end{align*}
 The assertion follows, as the result holds for each $x\in[0,1]^{d}$.
\end{proof}
The statement of Theorem~\ref{thm:TaylorBound} above applies for
general translation invariant kernels of the shape $k(x,y)=\phi(\left\Vert x-y\right\Vert ^{2})$.
To further specify the associated error bound, one needs to include
knowledge about the decay of the Taylor coefficients of $\phi$. To
this end we now consider a fixed kernel for which these coefficients
and their behavior are known.

\section{\label{sec:Gauss} The Gaussian kernel}

The following results build on reproducing kernel Hilbert spaces (RKHS).
For these spaces, point evaluations are continuous linear functionals,
and this property is the decisive characteristic of RKHS. Each kernel
functions addressed above is associated with the corresponding space
$\bigl(\mathcal{H}_{k},\,\langle\cdot|\cdot\rangle_{k}\bigr)$, for
which $\langle f|\,k(x,\cdot)\rangle_{k}=f(x)$ whenever $f\in\mathcal{H}_{k}$. For a detailed review of these spaces we refer to \citet{Berlinet2004} or \citet{steinwart2008support}.

This section addresses the most popular kernel in machine learning,
the Gaussian kernel
\begin{equation}
k(x,y)\coloneqq e^{-\sigma\left\Vert x-y\right\Vert _{2}^{2}}=\phi\bigl(\sigma\cdot\left\Vert x-y\right\Vert _{2}^{2}\bigr),\label{eq:Gaussian}
\end{equation}
where $\sigma>0$ is a width parameter. We approximate the point evaluation
function in the range of the Hilbert\textendash Schmidt operator and
analyze its error with respect to different norms. Building on these
estimates, we derive essential properties of the problem 
\begin{equation}
\inf_{w\in\mathcal{H}_{k}}\lambda\|w\|_{2}^{2}+\|L_{k}w-k_{x}\|_{k}^{2},\label{eq:MiniProbRKHS}
\end{equation}
which will be used in what follows. We employ these results to establish
polynomial bounds on the magnitude of the eigenfunctions of the Gaussian
kernel. 

\subsection{\label{subsec:4.1}Approximation of the point evaluation function}

In this section we investigate the approximation of the point evaluation
function $k_{x}$ for the Gaussian kernel. We relate the function
$k_{x}$ with an approximation from the image of the corresponding
integral operator $L_{k}$ and provide bounds with respect to the
infinity norm as well as the norm of the RKHS. The following theorem
provides the result for the uniform norm first. 
\begin{thm}[Uniform approximation of $k_{x}$ in $\left\Vert \cdot\right\Vert _{\infty}$]
\label{thm:d-dim-gauss} Let $k$ be the $d$-dimensional Gaussian
kernel with width parameter~$\sigma$. Setting $c_{\sigma}\coloneqq\max\left\{ 1,2e\sigma d\right\} $,
$c_{p}=\sup_{z\in[0,1]^{d}}p(z)^{-1}$ and
\[
C(\sigma,m)\coloneqq\frac{1}{1-\frac{\sigma ed}{\left\lfloor \frac{m}{2}\right\rfloor }},
\]
the uniform bound 
\[
\sup_{x\in[0,1]^{d}}\left\Vert L_{k}W_{m}^{x}-k_{x}\right\Vert _{\infty}\le(1+c_{p}^{\nicefrac{1}{2}}m^{d})\,C(\sigma,m)\left(\left\lfloor \frac{m}{2}\right\rfloor \frac{1}{\sigma ed}\right)^{\left\lfloor \frac{m}{2}\right\rfloor }
\]
holds for $W_{m}^{x}$ defined in~\eqref{eq:PointApprox} for $m>c_{\sigma}+1$.

Specifically, for $m(s)\coloneqq3c_{\sigma}s+2$, we have that 
\begin{equation}
\sup_{x\in[0,1]^{d}}\left\Vert L_{k}W_{m(td)}^{x}-k_{x}\right\Vert _{\infty}\le3(t\,d)^{-3t\,d}\label{eq:Gauss-1}
\end{equation}
whenever $t\ge\max\left\{ \frac{\ln(3)+(d-1)\ln(2)+\frac{1}{2}\ln(c_{p})+d\ln(3c_{\sigma}d)}{2d\ln(3)},\ 1\right\} .$
\end{thm}

\begin{proof}
We defer the proof to Appendix~\ref{sec:App-Gaussian-Approximation}.
\end{proof}
The uniform bound~\eqref{eq:Gauss-1} relates $L_{k}W_{m}^{x}$ and
$k_{x}$ for all points $x\in[0,1]^{d}$. However, the uniform norm
is slightly too weak when studying the objective~\eqref{eq:MiniProbRKHS},
which relies on the approximation in RKHS norm. To overcome this issue
we extend the result obtained and establish a dedicated bound, similar
to~\eqref{eq:Gauss-1}, but with respect to the stronger norm $\|\cdot\|_{k}$.
The next proposition reveals the desired bound. 
\begin{prop}[Uniform approximation of $k_{x}$ in the norm $\left\Vert \cdot\right\Vert _{k}$]
\label{prop:ApproxHk} Let $m=m(s)=3c_{\sigma}s+2$ and $t\ge c\coloneqq\max\{c_{0},c_{1},c_{2},1\}$,
where
\begin{align*}
c_{0} & =\frac{\ln(3)+(d-1)\ln(2)+\frac{1}{2}\ln(c_{p})+d\ln(3c_{\sigma}d)}{2d\ln(3)}.\\
c_{1} & =\frac{(2d-1)\ln(2)+d\ln(3c_{\sigma})+\frac{1}{2}\ln(c_{p})}{d}+1,\\
c_{2} & =\frac{(2d-1)\ln(2)+\frac{1}{2}\ln(c_{p})}{d}.
\end{align*}
In the setting of Theorem~\eqref{thm:d-dim-gauss}, the uniform bound
in $\|\cdot\|_{k}$-norm,
\begin{equation}
\sup_{x\in[0,1]^{d}}\left\Vert L_{k}W_{m(td)}^{x}-k_{x}\right\Vert _{k}^{2}\le9(td)^{-2td},\label{eq:RkhsEst}
\end{equation}
holds true.
\end{prop}

\begin{proof}
Let $x\in[0,1]^{d}$ and observe from~\eqref{eq:Gauss-1} and the
reverse triangle inequality that 
\begin{equation}
(L_{k}W_{m}^{x})(x)\ge k(x,x)-\left\Vert L_{k}W_{m}^{x}-k_{x}\right\Vert _{\infty}\ge k(x,x)-3(td)^{-3td},\label{eq:Pointwise}
\end{equation}
whenever $m$ and $t$ are chosen appropriately (see Theorem~\ref{thm:d-dim-gauss}).
Hence, setting $m=\left\lceil 3c_{\sigma}td\right\rceil +2$, it follows
from~\eqref{eq:Pointwise} that
\begin{align*}
\left\Vert L_{k}W_{m}^{x}-k_{x}\right\Vert _{k}^{2} & =\left\langle L_{k}W_{m}^{x},L_{k}W_{m}^{x}\right\rangle _{k}-2\left\langle L_{k}W_{m}^{x},k_{x}\right\rangle _{k}+k(x,x)\\
 & =\left\langle L_{k}W_{m}^{x},L_{k}W_{m}^{x}\right\rangle _{k}-2\left(L_{k}W_{m}^{x}\right)(x)+k(x,x)\\
 & \le\left\langle L_{k}W_{m}^{x},W_{m}^{x}\right\rangle _{L^{2}}-\left(L_{k}W_{m}^{x}\right)(x)+3(td)^{-3td},
\end{align*}
as $\left\langle L_{k}W_{m}^{x},k_{x}\right\rangle _{k}=\left(L_{k}W_{m}^{x}\right)(x)$
holds by the reproducing property. Furthermore, we have that
\begin{align*}
\left\langle L_{k}W_{m}^{x},W_{m}^{x}\right\rangle _{L^{2}} & =\left\langle L_{k}W_{m}^{x}-k_{x},W_{m}^{x}\right\rangle _{L^{2}}+\left\langle k_{x},W_{m}^{x}\right\rangle _{L^{2}}=\left\langle L_{k}W_{m}^{x}-k_{x},W_{m}^{x}\right\rangle _{L^{2}}+\left(L_{k}W_{m}^{x}\right)(x).
\end{align*}
Employing the bound
\[
\left\langle L_{k}W_{m}^{x}-k_{x},W_{m}^{x}\right\rangle _{L^{2}}\le6(dt)^{-2dt}
\]
from Lemma~\ref{eq:NormandWeights} in the appendix and combining
the estimates above, we get 
\[
\left\Vert L_{k}w_{m}^{x}-k_{x}\right\Vert _{k}^{2}\le6(td)^{-2td}+3(td)^{-3td}\le9(td)^{-2td},
\]
and thus the assertion.
\end{proof}
\begin{rem}
The bound~\eqref{eq:RkhsEst} depends only on the magnitude of the
product $td$. Substituting $s=td$ gives the more convenient bound
\begin{equation}
\sup_{x\in[0,1]^{d}}\left\Vert L_{k}W_{m}^{x}-k_{x}\right\Vert _{k}^{2}\le9s{}^{-2s},\label{eq:HandyEst}
\end{equation}
where $m=m(s)=3c_{\sigma}s+2$ and $s\ge d\cdot c.$
\end{rem}

\subsection{Implications for the weight function}

Building on the bound~\eqref{eq:RkhsEst} of Section~\ref{subsec:4.1},
we now examine the optimal weight function $w_{\lambda}^{x}$ solving
the optimization problem

\[
\inf_{w\in L^{2}}\lambda\left\Vert w\right\Vert _{L^{2}}^{2}+\left\Vert L_{k}w-k_{x}\right\Vert _{k}^{2},
\]
cf.~\eqref{eq:MiniProbRKHS}. This optimal solution $w_{\lambda}^{x}$
and its norm determines the approximation quality of the point evaluation
$k_{x}$ in the range of $L_{k}$. Moreover, they relate the continuous
operator~$L_{k}$ with discrete versions, which we address and discuss
in the following sections. 

By Mercer's theorem, the kernel $k$ enjoys the representation $k(x,y)=\sum_{\ell=1}^{\infty}\mu_{\ell}\,\varphi_{\ell}(x)\varphi_{\ell}(y)$,
where $\mu_{\ell}$ are the eigenvalues (in non-increasing order)
and $\varphi_{\ell}$ the eigenfunctions, $\ell=1,2,\dots$, of the
operator $L_{k}$ introduced in~\eqref{eq:Mercer}. The explicit
representation 
\[
w_{\lambda}^{x}(y)=\bigl((\lambda+L_{k})^{-1}L_{k}k_{x}\bigr)(y)=\sum_{\ell=1}^{\infty}\frac{\mu_{\ell}}{\lambda+\mu_{\ell}}\varphi_{\ell}(x)\varphi_{\ell}(y)
\]
of the optimal solution involves the regularization parameter~$\lambda$
and the components of Mercer's decomposition of the kernel $k$ (cf. \citealt[Proposition 8.6]{Cucker2007}). 

The next theorem provides a bound on the magnitude of the worst case
norm, more specifically, a bound for $\sup_{x\in[0,1]^{d}}\left\Vert w_{\lambda}^{x}\right\Vert _{L^{2}}$.
\begin{thm}[Uniform bound of the weight function]
\label{thm:TheWunction}Let $k$ be the Gaussian kernel. The weight
function $w_{\lambda}^{x}$ satisfies the bound
\begin{equation}
\sup_{x\in[0,1]^{d}}\left\Vert w_{\lambda}^{x}\right\Vert _{L^{2}}^{2}\le9c_{p}\left(3c_{\sigma}s(\lambda)+2\right)^{2d}+1\label{eq:NormGrowthwlambda}
\end{equation}
with $s(\lambda)\coloneqq\max\left\{ -\frac{1}{2}\ln\left(\frac{\lambda}{9}\right),\,d\cdot c,\,e\right\} $
and $c_{\sigma}$, $c$ from Proposition~\ref{prop:ApproxHk}. 
\end{thm}

\begin{proof}
The function $w_{\lambda}^{x}$ minimizes~\eqref{eq:MiniProbRKHS}
and thus
\begin{equation}
\sup_{x\in[0,1]^{d}}\lambda\left\Vert w_{\lambda}^{x}\right\Vert _{L^{2}}^{2}\le\sup_{x\in[0,1]^{d}}\lambda\left\Vert W_{m}^{x}\right\Vert _{L^{2}}^{2}+\left\Vert L_{k}W_{m}^{x}-k_{x}\right\Vert _{k}^{2}\le\sup_{x\in[0,1]^{d}}\lambda c_{p}m^{2d}+\left\Vert L_{k}W_{m}^{x}-k_{x}\right\Vert _{k}^{2},\label{eq:EstOpt}
\end{equation}
as $\left\Vert W_{m}^{x}\right\Vert _{2}^{2}\le c_{p}m^{2d}$ (see
Theorem~\ref{prop:MomDim}).

Choosing $s(\lambda)\coloneqq\max\left\{ -\frac{1}{2}\ln\left(\frac{\lambda}{9}\right),\,d\cdot c,\,e\right\} $
and $m=3c_{\sigma}\,s(\lambda)+2$ we derive from~\eqref{eq:HandyEst}
that 
\[
\left\Vert L_{k}W_{m}^{x}-k_{x}\right\Vert _{k}^{2}\le9s^{-2s}\le\lambda
\]
(as $s\ge e$ and $s\ge-\frac{1}{2}\ln\left(\frac{\lambda}{9}\right)$),
and thus
\begin{align*}
\sup_{x\in[0,1]^{d}}\lambda\left\Vert w_{\lambda}^{x}\right\Vert _{L^{2}}^{2} & \le\sup_{x\in[0,1]^{d}}\lambda c_{p}m^{2d}+\left\Vert L_{k}W_{m}^{x}-k_{x}\right\Vert _{k}^{2}\le\lambda c_{p}(3c_{\sigma}s(\lambda)+2)^{2d}+\lambda
\end{align*}
by~\eqref{eq:EstOpt}. Multiplying with $\lambda^{-1}$ gives the
assertion. 
\end{proof}
The bound~\eqref{eq:NormGrowthwlambda} characterizes the asymptotic
growth of the norm $\left\Vert w_{\lambda}^{x}\right\Vert _{L^{2}}$.
Indeed, letting $\lambda\downarrow0$, the bound~\eqref{eq:NormGrowthwlambda}
implies that $\left\Vert w_{\lambda}^{x}\right\Vert _{L^{2}}$ grows
at most as $\bigl(\ln(\lambda^{-1})\bigr)^{d}$. The subsequent sections
heavily exploit this asymptotic behavior. 

\subsection{Eigenvalues and eigenfunctions}

This section studies the elements in Mercer's decomposition of the
kernel, $k(x,y)=\sum_{\ell=1}^{\infty}\mu_{\ell}\varphi_{\ell}(x)\varphi_{\ell}(y)$,
specifically for the Gaussian kernel. We first relate the maximal
value of any eigenfunction with its associated eigenvalue, demonstrating
that $\left\Vert \varphi_{\ell}\right\Vert _{\infty}$ is bounded
by $\bigl(\ln(\mu_{\ell}^{-1})\bigl)^{d}$. This bound is significantly
sharper than the standard bound $\max_{x\in\mathcal{X}}\varphi_{\ell}(x)\le k(x,x)^{\nicefrac{1}{2}}\mu_{\ell}^{\nicefrac{-1}{2}}$
derived from the Cauchy\textendash Schwartz inequality, $\langle\varphi,\,k_{x}\rangle_{k}\le\|\varphi\|_{k}\,k(x,x)$. 

We next describe the decay of the eigenvalues $\mu_{\ell}$ and derive
a bound on $\left\Vert \varphi_{\ell}\right\Vert _{\infty}$, which
turns out to be quadratic in~$\ell$. The results of this section enable us to
infer convergence in uniform norm from convergence in $L^{2}$. Generally,
this approach leads to faster convergence rates compared to results,
which are derived from convergence in the norm $\left\Vert \cdot\right\Vert _{k}$
(see Section~\ref{subsec:Pointwise}).

Our approach builds on the following observation. For the regularization
$\lambda=\mu_{\ell}$, the series 
\[
\left\Vert w_{\lambda}^{x}\right\Vert _{L^{2}}^{2}=\sum_{h=1}^{\infty}\frac{\mu_{h}}{\lambda+\mu_{h}}\varphi_{h}^{2}(x)
\]
involves the term $\frac{1}{2}\varphi_{\ell}^{2}(x)$. To bound the
maximum of $\left|\varphi_{\ell}\right|$, it is sufficient to assess
$\sup_{x\in[0,1]^{d}}\left\Vert w_{\lambda}^{x}\right\Vert _{L^{2}}^{2}$,
which is bounded by the inequality~\eqref{eq:NormGrowthwlambda}.
The following result summarizes these relations.
\begin{thm}
For the eigenfunctions of the Gaussian kernel, the inequality
\begin{equation}
\sup_{x\in[0,1]^{d}}\left|\varphi_{\ell}(x)\right|\le\sqrt{18c_{p}}\left(3c_{\sigma}s(\mu_{\ell})+2\right)^{d}+\sqrt{2}\label{eq:BoundEigenfunc}
\end{equation}
holds for every $\ell\in\mathbb{N}$, where $s(\mu_{\ell})=\max\left\{ -\frac{1}{2}\ln\left(\frac{\mu_{\ell}}{9}\right),\ e,\ d\cdot c\right\} $
is as in Theorem~\ref{thm:TheWunction}.
\end{thm}

\begin{proof}
For $\ell\in\mathbb{N}$, it holds that
\[
\frac{1}{2}\left|\varphi_{\ell}(x)\right|^{2}=\frac{\mu_{\ell}}{\mu_{\ell}+\mu_{\ell}}\varphi_{\ell}(x)^{2}\le\sum_{h=1}^{\infty}\frac{\mu_{h}}{\mu_{\ell}+\mu_{h}}\varphi_{h}(x)^{2}=\left\Vert w_{\lambda^{\ast}}^{x}\right\Vert _{L^{2}}^{2},
\]
with $\lambda^{\ast}=\mu_{\ell}$. Taking the maximum on both sides,
we deduce from the latter inequality and~\eqref{eq:NormGrowthwlambda}
that 
\begin{align*}
\frac{1}{2}\left\Vert \varphi_{\ell}\right\Vert _{\infty}^{2} & \le\sup_{x\in[0,1]^{d}}\left\Vert w_{\lambda^{\ast}}^{x}\right\Vert _{L^{2}}^{2}\\
 & \le9c_{p}\left(3c_{\sigma}s(\lambda^{\ast})+2\right)^{2d}+1=9c_{p}\left(3c_{\sigma}s(\mu_{\ell})+2\right)^{2d}+1.
\end{align*}
Reformulating this inequality as well as using the subadditivity of
the square root gives the assertion. 
\end{proof}
The bound~\eqref{eq:BoundEigenfunc} relates the eigenfunctions and
the eigenvalues of the operator $L_{k}$. In what follows, we analyze
the decay of $(\mu_{\ell})_{\ell=1}^{\infty}$ to get a more concrete
characterization of the bound in~\eqref{eq:BoundEigenfunc}. 

The next lemma provides a lower bound of the eigenvalues $(\mu_{\ell})_{\ell=1}^{\infty}$.
\begin{lem}[Maximal decay of eigenvalues]
\label{lem:Eigdec}Set $p_{\min}\coloneqq\inf_{x\in\mathcal{X}}p(x)$
and $p_{\max}\coloneqq\sup_{x\in\mathcal{X}}p(x)$. For the eigenvalues
of the Gaussian kernel it holds that
\begin{equation}
\mu_{\ell}\ge\frac{p_{\min}^{2}}{p_{\max}^{2}}C(d,\sigma)e^{-c_{\sigma,d}(\ell+d)^{\frac{2}{d}}},\qquad\ell=1,2,\dots,\label{eq:Eigdecay-2-1}
\end{equation}
where the constants $c_{d,\sigma}$ and $C(d,\sigma)$ depend on the
dimension~$d$ and the bandwidth~$\sigma$.
\end{lem}

\begin{proof}
See Appendix~\ref{sec:Eigenvalue-Approximation}.
\end{proof}
The subsequent theorem combines the inequalities~\eqref{eq:BoundEigenfunc}
and~\eqref{eq:Eigdecay-2-1} and provides an explicit bound on the
maximal absolute value of the eigenfunctions $(\varphi_{\ell})_{\ell=1}^{\infty}.$ 
\begin{thm}[Eigenfunctions]
 The eigenfunctions of the $d$-dimensional Gaussian kernel satisfy
the bound 
\begin{equation}
\max_{x\in[0,1]^{d}}\left|\varphi_{\ell}(x)\right|\le\sqrt{18c_{p}}\max\left\{ h(\ell),\ \left(3c_{\sigma}e+2\right)^{d},\left(3c_{\sigma}(d\cdot c)+2\right)^{d}\right\} +\sqrt{2}\label{eq:EigBound}
\end{equation}
with 
\[
h(\ell)\coloneqq\left(3c_{\sigma}\left(\left(\frac{1}{2}\ln\left(9\right)-\frac{1}{2}\ln\left(\frac{p_{\min}^{2}}{p_{\max}^{2}}C(d,\sigma)\right)+\frac{1}{2}c_{d,\sigma}(\ell+d)^{\frac{2}{d}}\right)\right)_{+}+2\right)^{d}
\]
and the constants from the preceding Lemma~\ref{lem:Eigdec}. 
\end{thm}

\begin{proof}
For the assertion note first that 
\[
s(\lambda)=\max\left\{ -\frac{1}{2}\ln\left(\frac{\lambda}{9}\right),\ e,\ d\cdot c\right\} =\max\left\{ \left(-\frac{1}{2}\ln\left(\frac{\lambda}{9}\right)\right)_{+},\ e,\ d\cdot c\right\} ,
\]
where $x_{+}\coloneqq\max\left\{ x,\ 0\right\} $ is the positive
part of $x$. Hence, employing~\eqref{eq:Eigdecay-2-1} in~\eqref{eq:BoundEigenfunc},
we have the bound
\begin{align}
\left|\varphi_{\ell}(x)\right| & \le\sqrt{18c_{p}}\left(3c_{\sigma}s(\mu_{\ell})+2\right)^{d}+\sqrt{2}\nonumber \\
 & =\sqrt{18c_{p}}\left(3c_{\sigma}\max\left\{ -\frac{1}{2}\ln\left(\frac{p_{\min}^{2}C(d,\sigma)e^{-c_{\sigma,d}(\ell+d)^{\frac{2}{d}}}}{9p_{\max}^{2}}\right),\ e,\ d\cdot c\right\} +2\right)^{d}+\sqrt{2}\nonumber \\
 & =\sqrt{18c_{p}}\left(3c_{\sigma}\max\left\{ \left(\frac{1}{2}\ln\left(9\right)-\frac{1}{2}\ln\left(\frac{p_{\min}^{2}}{p_{\max}^{2}}C(d,\sigma)\right)+\frac{1}{2}c_{d,\sigma}(\ell+d)^{\frac{2}{d}}\right)_{+},\,e,\,d\cdot c\right\} +2\right)^{d}\label{eq:Finalest}\\
 & \qquad+\sqrt{2}.\nonumber 
\end{align}
Note now that $x\mapsto x^{d}$ is an increasing function on $[0,\infty)$
and thus $\max\left\{ a,b\right\} ^{d}=\max\{a^{d},b^{d}\}$, provided
that $a,b\ge0.$ Using this property in~\eqref{eq:Finalest} provides
the assertion~\eqref{eq:EigBound}.
\end{proof}
\begin{rem}
The right-hand side of \eqref{eq:EigBound} deserves additional attention.
While $h(\ell)$ grows with~$\ell$ increasing, the other terms do
not depend on~$\ell$. For sufficiently large $\ell$, the inequality~\eqref{eq:EigBound}
thus reads 
\[
\max_{x\in[0,1]^{d}}\left|\varphi_{\ell}(x)\right|\le18c_{p}^{\nicefrac{1}{2}}h(\ell)+\sqrt{2}.
\]
The function $h$ itself grows at most as a polynomial with degree
$2$. That is, there exists a constant $b>0$ with quadratic bound
\begin{equation}
\max_{x\in[0,1]^{d}}\left|\varphi_{\ell}(x)\right|\le b\,\ell^{2},\qquad\ell=1,2,\dots.\label{eq:main}
\end{equation}

To the best of our knowledge, this is the first non-exponential bound
of $\varphi_{\ell}'s$ absolute maximum. The following section addresses
various consequences of this main result.
\end{rem}

\begin{rem}
The approach chosen in this Section~\ref{sec:Gauss} is not limited
to the Gaussian kernel. All steps extend to general kernels of the
form $\phi(\left\Vert x-y\right\Vert ^{2})$ (cf.~\eqref{eq:nonGauss}),
provided that a lower bound on the decay of its eigenvalues is available. 
\end{rem}

\section{\label{sec:Implications}Main results}

This section connects the results of the preceding sections for general
statistical learning settings. We present improved concentration bound
inequalities, an interpolation inequality relating to uniform convergence,
and the Nyström method. 

We adopt the well-established notation (cf.~\citealt{NIPS2015_03e0704b})
and introduce 
\begin{equation}
\mathcal{N}_{\infty}(\lambda)\coloneqq\sup_{x\in[0,1]^{d}}\left\Vert w_{\lambda}^{x}\right\Vert _{L^{2}}^{2}.\label{eq:Ninf}
\end{equation}
As established above in~\eqref{eq:NormGrowthwlambda}, it holds that
\[
\mathcal{N}_{\infty}(\lambda)=\mathcal{O}\bigl(\ln(\lambda^{-1})^{2d}\bigr)
\]
for $\lambda\to0$. 

Building on the results of Section~\ref{sec:Gauss}, we provide the
results for the Gaussian kernel.

\subsection{\label{subsec:Concentration} Concentration bounds}

Standard kernel ridge regression minimizes the regularized squared
error at the sample points. To infer the related error at other locations,
one needs to connect the discrete setting with the continuous setting
from Section~\ref{sec:Gauss}. This is commonly done by relating
the discrete operator 
\[
L_{k}^{D}\colon C(\mathcal{X})\to\mathcal{H}_{k}\hspace{1em}\text{with}\hspace{1em}(L_{k}^{D}f)(y)\coloneqq\frac{1}{n}\sum_{i=1}^{n}f(x_{i})k(y,x_{i})
\]
with its continuous version $L_{k}$, as demonstrated in \citet{Caponnetto2006}
as well as in \citet{steinwartfischer2020}. However, the underlying
concentration results in these references generally require regularization
sequences not faster than $\lambda_{n}=\mathcal{O}(\nicefrac{1}{n})$,
which is too restrictive for many tasks including the analysis of
low rank kernel methods (cf.\ \citealt{Zhang2013}, \citealt{Bach2013}). 

This section addresses this issue. We provide the following result,
which allows significantly higher flexibility in choosing the regularization
sequence $\lambda_{n}$. Specifically, the regularizing sequence may
be chosen to decay faster than any polynomial.

The following proposition collects the detailed result for the concentration
inequality. 
\begin{prop}
\label{prop:Concentration} Assume that~$\lambda$ satisfies 
\begin{equation}
\frac{4}{3}\tau\,g(\lambda)\frac{\mathcal{N}_{\infty}(\lambda)}{n}+\sqrt{2\tau\,g(\lambda)\frac{\mathcal{N}_{\infty}(\lambda)}{n}}\le\frac{1}{2}\label{eq:LambdaCond}
\end{equation}
with 
\[
g(\lambda)=\ln\left(2e\frac{(\lambda+\mu_{1})\mathcal{N}(\lambda)}{\mu_{1}}\right), \quad \mathcal{N}(\lambda)=\sum_{\ell=1}^{\infty} \frac{\mu_{\ell}}{\lambda + \mu_{\ell}}
\]
and $\tau>0$. Then the inequality
\begin{align}
\left\Vert (\lambda+L_{k}^{D})^{-1}(L_{k}+\lambda)\right\Vert _{\mathcal{H}_{k}\to\mathcal{H}_{k}} & \le2\label{eq:ConcInequExakt-1}
\end{align}
holds with probability at least $1-2e^{-\tau}$.
\end{prop}

\begin{proof}
For the proof we refer to Appendix~\ref{sec:ConcentrationAPP}.
\end{proof}
The condition~\eqref{eq:LambdaCond} deserves some additional commentary.
The term $\mathcal{N}_{\infty}(\lambda)$ is asymptotically bounded
by $\mathcal{O}\bigl(\ln(\lambda^{-1})^{2d}\bigr)$, and the function
$g(\lambda)$ does not grow faster than $\mathcal{O}\bigl(\ln(\lambda^{-1})\bigr)$.
The condition~\eqref{eq:LambdaCond} is asymptotically satisfied,
provided that the regularization parameter $\lambda$ does not decay
faster than $c\exp\bigl(-n^{\frac{1}{2d+1}}\bigr)$. This is a significant
improvement compared to the common regularization choice $n^{-1}$.

The conclusion of Proposition~\ref{prop:Concentration} can be given
more general, not necessarily involving the Gaussian kernel. Moreover,
the bound~\eqref{eq:NormGrowthwlambda} of $\mathcal{N}_{\infty}(\lambda)$
only relies on the decay of the Taylor coefficients associated with
the kernel. The same method thus may be applied to derive concentration
inequalities for any other radial kernel functions.

\subsection{Interpolation inequality \label{subsec:Pointwise}}

This section establishes results on the uniform approximation quality
for smooth functions. The following interpolation inequality ensures
that the uniform norm of smooth functions is comparable its norm in
$L^{2}$, although this norm is much weaker in general. We measure
smoothness with reference to the norm $\|\cdot\|_{s}$ introduced
below. 
\begin{thm}[Interpolation of norm]
\label{thm:main}\label{thm:Weights}For $f=\sum_{\ell=1}^{\infty}c_{\ell}\varphi_{\ell}\in L^{2}$
with $\left\Vert f\right\Vert _{s}^{2}\coloneqq\sum_{\ell=1}^{\infty}\frac{c_{\ell}^{2}}{\ell^{-2s}}<\infty$
it holds that 
\begin{equation}
\left\Vert f\right\Vert _{\infty}\le\frac{\pi}{\sqrt{6}}b\left\Vert f\right\Vert _{s}^{\frac{3}{s}}\cdot\left\Vert f\right\Vert _{2}^{1-\frac{3}{s}},\label{eq:PointwiseRelation}
\end{equation}
where $s>3$ and $\max_{x\in\mathcal{X}}\left|\varphi_{\ell}(x)\right|\le b\,\ell^{2}$
for all $\ell=1,2,\dots$ (cf.~\eqref{eq:main}).
\end{thm}

\begin{proof}
Building on the bound $\max_{x\in\mathcal{X}}\left|\varphi_{\ell}(x)\right|\le b\ell^{2}$
we have with the Cauchy\textendash Schwarz inequality that 
\begin{align*}
f(x)=\sum_{\ell=1}^{\infty}c_{\ell}\varphi_{\ell}(x) &\le b\sum_{\ell=1}^{\infty}\frac{c_{\ell}}{\ell^{-\frac{s}{p}}}\cdot\ell^{2-\frac{s}{p}}\\ &\le b\sqrt{\sum_{\ell=1}^{\infty}\frac{c_{\ell}^{2}}{\ell^{-\frac{2s}{p}}}}\sqrt{\sum_{\ell=1}^{\infty}\ell^{(2-\frac{s}{p})2}}= b\sqrt{\sum_{\ell=1}^{\infty}\frac{c_{\ell}^{\frac{2}{p}}}{\ell^{-\frac{2s}{p}}}c_{\ell}^{2-\frac{2}{p}}}\,\sqrt{\zeta\Bigl(\frac{2s}{p}-4\Big)},
\end{align*}
where $\zeta(\cdot)$ is the Riemann zeta function and $p\in[1,2s]$.
Employing Hölder?s inequality, it follows for the first term that
\begin{align*}
\sqrt{\sum_{\ell=1}^{\infty}\frac{c_{\ell}^{\frac{2}{p}}}{\ell^{-\frac{2s}{p}}}c^{2-\frac{2}{p}}} & \le\left(\sum_{\ell=1}^{\infty}\frac{c_{\ell}^{2}}{\ell^{-2s}}\right)^{\frac{1}{2p}}\left(\sum_{\ell=1}^{\infty}c_{\ell}^{(2-\frac{2}{p})(\frac{p}{p-1})}\right)^{\frac{1}{2}(1-\frac{1}{p})}\\
 & =\left(\sum_{\ell=1}^{\infty}\frac{c_{\ell}^{2}}{\ell^{-2s}}\right)^{\frac{1}{2p}}\left(\sum_{\ell=1}^{\infty}c_{\ell}^{2}\right)^{\frac{1}{2}(1-\frac{1}{p})}.
\end{align*}
Choosing $p=\frac{s}{3}$, we finally have that
\begin{align*}
b\left(\sum_{\ell=1}^{\infty}\frac{c_{\ell}^{2}}{\ell^{-2s}}\right)^{\frac{1}{2p}}\left(\sum_{\ell=1}^{\infty}c_{\ell}^{2}\right)^{\frac{1}{2}(1-\frac{1}{p})}\sqrt{\zeta\Bigl(\frac{2s}{p}-4\Big)} & =b\left\Vert f\right\Vert _{s}^{\frac{1}{p}}\left\Vert f\right\Vert _{2}^{1-\frac{1}{p}}\sqrt{\zeta\Bigl(\frac{2s}{p}-4\Big)} \\ &=\frac{\pi}{\sqrt{6}}b\left\Vert f\right\Vert _{s}^{\frac{3}{s}}\left\Vert f\right\Vert _{2}^{\frac{s-3}{s}}
\end{align*}
by involving Euler's famous formula $\zeta(2)=\frac{\pi^{2}}{6}$.
This is the assertion.
\end{proof}

\paragraph{Convergence in $L^{2}$ implies convergence in $L^{\infty}$:}

Specifically, consider a sequence of functions $f_{n}$ with slowly
increasing norm $\|f_{n}\|_{s}$ and limit $\|f\|_{s}<\infty$. Then,
for $\|f_{n}-f\|_{2}\to0$ sufficiently fast, it follows from Theorem~\ref{thm:main}
that $\|f_{n}-f\|_{\infty}\to0$. 
\begin{rem}
\citet{steinwartfischer2020} study convergence in an interpolation
space between $\mathcal{H}_{k}$ and $L^{2}$, which the authors call
\emph{power spaces}. The norm considered in \citet{steinwartfischer2020},
however, is stronger than the norm $\|\cdot\|_{s}$ considered above.
\end{rem}

\subsection{Nyström method}

The main result in \citet[Theorem~1]{NIPS2015_03e0704b} relates the
function $\mathcal{N}_{\infty}(\lambda)$ to the Nyström method. Their
paper ensures that the Nyström method does not require more than $\mathcal{O}\left(\mathcal{N}_{\infty}(\lambda)\ln\lambda^{-1}\right)$
supporting points. 

The results established in Section~\ref{sec:Gauss} above provides
explicit access to the function $\mathcal{N}_{\infty}(\lambda)$,
so that their main result can be refined and enhanced to the following
form.
\begin{thm}[{Cf.~\citealt[Theorem~1]{NIPS2015_03e0704b}}]
Let $\mathcal{E}(f)\coloneqq\iint_{\mathcal{X}\times\mathbb{R}}\bigl(f(x)-y\bigr)^{2}\rho(dx,dy)$
be the common error function. Given the assumptions of Theorem~1
in \citep{NIPS2015_03e0704b}, it holds that the Nyström-approximation
$\hat{f}_{\lambda,m}$ with regression parameter $\lambda$ satisfies
\[
\mathcal{E}(\hat{f}_{\lambda,m})-\mathcal{E}(f_{\mathcal{H}})\le q^{2}n^{\frac{2\nu+1}{2nu+\gamma+1}}
\]
for at least
\begin{equation}
m\ge(67\lor5\left(9c_{p}\left(3c_{\sigma}s(\lambda)+2\right)^{2d}+1\right)\ln\lambda^{-1}\label{eq:1}
\end{equation}
supporting points. The reference provides the constant $q$ and $\nu$
explicitly.
\end{thm}

\begin{proof}
Invoking the bound~\eqref{eq:NormGrowthwlambda} for $\mathcal{N}_{\infty}(\lambda)$,
the result is immediate from \citet[Theorem~1]{NIPS2015_03e0704b}.
\end{proof}
The adapter result~\eqref{eq:1} gives an explicit selection criterion
for the critical number of support points in the Nyström method. 

\section{\label{sec:Conclusion}Summary}

The novel approach of this paper considers an explicit approximation
of the kernel function in the range of the associated integral operator.
To this end we provide an explicit weight function by matching the
initial Taylor coefficients of the kernel. 

The approach has numerous consequences in theory and in applications.
We provide bounds for the eigenfunctions, which grow only quadratic
in the enumeration index. An interpolation inequality provided relates
convergence in uniform norm and the weaker convergence in $L^{2}$
for smooth functions. The methods established justify smaller regression
parameters for regression problems, which is of particular importance
for low-rank approximation techniques. 

\bibliographystyle{abbrvnat}
\bibliography{PaulLibNew}

\appendix

\section{\label{sec:App-Gaussian-Approximation}Gaussian approximation}

This section provides the proof of the formula \eqref{eq:Gauss-1},
which is of fundamental importance for every other bound provided
in Section~\ref{sec:Gauss}. The proof of the bound \eqref{eq:Gauss-1}
builds on the error estimate \eqref{eq:TaylorErrorMoreDim}, which
involves the Taylor remainder of the exponential series. To this end
we utilize the formula
\begin{equation}
\sum_{\ell=n}^{\infty}\alpha^{\ell}=\frac{\alpha^{n}}{1-\alpha}\label{eq:Geometric}
\end{equation}
of the truncated geometric series as well as Stirling's approximation
\begin{equation}
\sqrt{2\pi n}\left(\frac{n}{e}\right)^{n}<n!,\label{eq:Sterling}
\end{equation}
which relates the factorial with the exponential function. For future
reference convenience of the reader, we restate Theorem~\ref{thm:d-dim-gauss}.
\begin{thm}
Let $k$ be the $d$-dimensional Gaussian kernel with width parameter~$\sigma$.
Setting $c_{\sigma}\coloneqq\max\left\{ 1,2e\sigma d\right\} $, $c_{p}=\sup_{z\in[0,1]^{d}}p(z)^{-1}$
and
\[
C(\sigma,m)\coloneqq\frac{1}{1-\frac{\sigma ed}{\left\lfloor \frac{m}{2}\right\rfloor }},
\]
the uniform bound 
\[
\sup_{x\in[0,1]^{d}}\left\Vert L_{k}W_{m}^{x}-k_{x}\right\Vert _{\infty}\le(1+c_{p}^{\nicefrac{1}{2}}m^{d})\,C(\sigma,m)\left(\left\lfloor \frac{m}{2}\right\rfloor \frac{1}{\sigma ed}\right)^{\left\lfloor \frac{m}{2}\right\rfloor }
\]
holds for $W_{m}^{x}$ defined in~\eqref{eq:PointApprox} for $m>c_{\sigma}+1$.

Specifically, for $m(s)\coloneqq3c_{\sigma}s+2$, we have that 
\begin{equation}
\sup_{x\in[0,1]^{d}}\left\Vert L_{k}W_{m(td)}^{x}-k_{x}\right\Vert _{\infty}\le3(t\,d)^{-3t\,d}\label{eq:Gauss-1-1}
\end{equation}
whenever $t\ge\max\left\{ \frac{\ln(3)+(d-1)\ln(2)+\frac{1}{2}\ln(c_{p})+d\ln(3c_{\sigma}d)}{2d\ln(3)},\ 1\right\} .$
\end{thm}

\begin{proof}
Employing the inequality~\eqref{eq:TaylorErrorMoreDim} involving
the Taylor coefficients, we have that
\begin{equation}
\sup_{x\in[0,1]^{d}}\left\Vert L_{k}W_{m}^{x}-k_{x}\right\Vert _{\infty}\le(1+c_{p}^{\nicefrac{1}{2}}m^{d})\sum_{\ell=\left\lfloor \frac{m-1}{2}\right\rfloor +1}^{\infty}\frac{\sigma^{\ell}}{\ell!}d^{\ell}\le(1+c_{p}^{\nicefrac{1}{2}}m^{d})\sum_{\ell=\left\lfloor \frac{m}{2}\right\rfloor }^{\infty}\frac{\sigma^{\ell}}{\ell!}d^{\ell},\label{eq:Est1-2}
\end{equation}
as $\left\lfloor \frac{m}{2}\right\rfloor \le\left\lfloor \frac{m-1}{2}\right\rfloor +1$.
Further, invoking Stirling's approximation \eqref{eq:Sterling} for the
right-hand side~of \eqref{eq:Est1-2}, it follows that
\begin{align*}
(1+c_{p}^{\nicefrac{1}{2}}m^{d})\sum_{\ell=\left\lfloor \frac{m}{2}\right\rfloor }^{\infty}\sigma^{\ell}\left(\ell!\right)^{-1}d^{\ell} & \le(1+c_{p}^{\nicefrac{1}{2}}m^{d})\sum_{\ell=\left\lfloor \frac{m}{2}\right\rfloor }^{\infty}\left(\sqrt{2\pi\ell}\left(\frac{\ell}{\sigma ed}\right)^{\ell}\right)^{-1}.
\end{align*}
Now assume that $m$ satisfies the inequality $1<\left\lfloor \frac{m}{2}\right\rfloor \left(\sigma ed\right)^{-1}$,
then we have from identity~\eqref{eq:Geometric} of the truncated
geometric series that
\begin{align}
\sup_{x\in[0,1]^{d}}\left\Vert L_{k}W_{m}^{x}-k_{x}\right\Vert _{\infty} & \le(1+c_{p}^{\nicefrac{1}{2}}m^{d})\sum_{\ell=\left\lfloor \frac{m}{2}\right\rfloor }^{\infty}\left(\frac{\ell}{\sigma ed}\right)^{-\ell}\nonumber \\
 & \le(1+c_{p}^{\nicefrac{1}{2}}m^{d})\frac{1}{1-\frac{\sigma ed}{\left\lfloor \frac{m}{2}\right\rfloor }}\left(\left\lfloor \frac{m}{2}\right\rfloor \frac{1}{\sigma ed}\right)^{-\left\lfloor \frac{m}{2}\right\rfloor }\\
 & =(1+c_{p}^{\nicefrac{1}{2}}m^{d})C(\sigma,m)\left(\left\lfloor \frac{m}{2}\right\rfloor \frac{1}{\sigma ed}\right)^{-\left\lfloor \frac{m}{2}\right\rfloor },\label{eq:Est1-1}
\end{align}
which is the first assertion. 

For the second assertion let $m=m(td)=6c_{\sigma}td+2$ and observe
for the constant $C(\sigma,m)$ that 
\begin{align*}
C(\sigma,m) & =\frac{1}{1-\frac{\sigma ed}{\left\lfloor \frac{m}{2}\right\rfloor }}\le\frac{1}{1-\frac{\sigma ed}{3(\sigma ed)td}}\le\frac{1}{1-\frac{1}{3td}}\le\frac{3}{2}
\end{align*}
holds for every $t\ge1$. Furthermore, employing $m(td)$, as well
as $C(\sigma,m)\le\frac{3}{2}$, and arguing as above, we get for
\eqref{eq:Est1-1} that
\begin{align*}
\MoveEqLeft[3](1+c_{p}^{\nicefrac{1}{2}}m^{d})C(\sigma,m)\left(\left\lfloor \frac{m}{2}\right\rfloor \frac{1}{\sigma ed}\right)^{-\left\lfloor \frac{m}{2}\right\rfloor }\\
 & \le\frac{3}{2}(1+c_{p}^{\nicefrac{1}{2}}m(t)^{d})\left(\frac{3\sigma ed^{2}t}{\sigma ed}\right)^{-\left\lfloor \frac{m}{2}\right\rfloor }\\
 & \le\frac{3}{2}(1+c_{p}^{\nicefrac{1}{2}}(3c_{\sigma}td+2)^{d})\left(3td\right)^{-3dt}\\
 & \le\frac{3}{2}\left(3td\right)^{-3dt}+\frac{3}{2}2^{d-1}c_{p}^{\nicefrac{1}{2}}(3c_{\sigma}td)^{d}\left(3td\right)^{-3dt}+\frac{3}{2}2^{2d-1}\left(3td\right)^{-3dt},
\end{align*}
where we utilize $(a+b)^{d}\le2^{d-1}(a^{d}+b^{d})$ for the last
inequality. The latter is bounded by 
\[
\frac{3}{2}2^{2d-1}\left(3td\right)^{-3dt}\le(td)^{-3dt}
\]
as well as the middle term by
\begin{align*}
\MoveEqLeft[5]2^{d-1}c_{p}^{\nicefrac{1}{2}}(3c_{\sigma}td)^{d}3^{-3td}(td)^{-3td}\\
 & =\frac{1}{2}(e^{\ln(3)+(d-1)\ln(2)+\ln(c_{p}^{\nicefrac{1}{2}})+d\ln(3c_{\sigma}td)-\ln(3)3td})(td)^{-3td}\\
 & \le\frac{1}{2}(e^{\ln(3)+(d-1)\ln(2)+\ln(c_{p}^{\nicefrac{1}{2}})+d\ln(3c_{\sigma}d)-\ln(3)2td})(dt)^{-3td}\\
 & \le\frac{1}{2}(td)^{-3td}
\end{align*}
whenever $t\ge\frac{\ln(3)+(d-1)\ln(2)+\ln(c_{p}^{\nicefrac{1}{2}})+d\ln(3\max\left\{ c_{\sigma},1\right\} d)}{2d\ln(3)}.$
Hence, choosing 
\[
t\ge\max\left\{ \frac{\ln(3)+(d-1)\ln(2)+\ln(c_{p}^{\nicefrac{1}{2}})+d\ln(3\max\left\{ c_{\sigma},1\right\} d)}{2d\ln(3)},1\right\} ,
\]
we have that 
\[
(1+c_{p}^{\nicefrac{1}{2}}m^{d})C(\sigma,m)\left(\left\lfloor \frac{m}{2}\right\rfloor \frac{1}{\sigma ed^{2}}\right)^{\left\lfloor \frac{m}{2}\right\rfloor }\le\frac{3}{2}\left(3dt\right)^{-3dt}+\frac{3}{2}(dt)^{-3dt}\le3(dt)^{-3dt},
\]
which is the assertion.
\end{proof}
Building on the bound \eqref{eq:Gauss-1} in the uniform norm, we
establish a bound in the RKHS norm in Proposition~\ref{prop:ApproxHk}.
To this end the following technical Lemma is of crucial importance.
\begin{lem}
\label{lem:NormandWeights}Given the assumptions of Theorem~\ref{thm:d-dim-gauss},
the bound 
\begin{equation}
\sup_{x\in[0,1]^{d}}\left\langle L_{k}W_{m}^{x}-k_{x},W_{m}^{x}\right\rangle _{L^{2}}\le6(td)^{-2td}\label{eq:NormandWeights}
\end{equation}
holds whenever $t\ge\max\left\{ c_{0},c_{1},c_{2}\right\} $ with
constants
\begin{align*}
c_{1} & =\frac{(2d-1)\ln(2)+d\ln(3c_{\sigma})+\frac{1}{2}\ln(c_{p})}{d}+1\hspace{1em}\text{and}\\
c_{2} & =\frac{(2d-1)\ln(2)+\frac{1}{2}\ln(c_{p})}{d}.
\end{align*}
\end{lem}

\begin{proof}
Employing the Cauchy\textendash Schwarz inequality, we have from \eqref{eq:Gauss-1}
that
\begin{align*}
\left\langle L_{k}W_{m}^{x}-k_{x},W_{m}^{x}\right\rangle _{2} & \le\left\Vert L_{k}W_{m}^{x}-k_{x}\right\Vert _{L^{2}}\left\Vert W_{m}^{x}\right\Vert _{L^{2}}\le3(td)^{-3td}c_{p}^{\nicefrac{1}{2}}m^{d}
\end{align*}
whenever $m$ (and $t$) are chosen with respect to the constraints
Theorem~\ref{thm:d-dim-gauss}. Involving $m=m(td)$ we further observe
that 
\begin{align}
3(td)^{-3td}c_{p}^{\nicefrac{1}{2}}(3c_{\sigma}td+2)^{d} & \le2^{d-1}3(td)^{-3td}c_{p}^{\nicefrac{1}{2}}(3c_{\sigma}td)^{d}+2^{2d-1}3(td)^{-3td}c_{p}^{\nicefrac{1}{2}}.\label{eq:4}
\end{align}
For the second term in \eqref{eq:4} it follows that
\[
2^{2d-1}3(td)^{-3td}c_{p}^{\nicefrac{1}{2}}=2^{2d-1}(td)^{-td}c_{p}^{\nicefrac{1}{2}}3(td)^{-2td}\le3(td)^{-2td}
\]
whenever $t\ge\frac{(2d-1)\ln(2)+\frac{1}{2}\ln(c_{p})}{d}\coloneqq c_{2}.$
Reformulating the first term in \eqref{eq:4} to
\[
2^{d-1}3(td)^{-3td}c_{p}^{\nicefrac{1}{2}}(3c_{\sigma}td)^{d}=3e^{-td\ln(td)+(2d-1)\ln(2)+\frac{1}{2}\ln(c_{p})+d\ln(3c_{\sigma}td)}(td)^{-2td}
\]
we get for the exponent that
\begin{align*}
 & -td\ln(td)+(d-1)\ln(2)+\frac{1}{2}\ln(c_{p})+d\ln(3c_{\sigma}td)\\
= & -d(t-1)\ln(td)+(d-1)\ln(2)+\frac{1}{2}\ln(c_{p})+d\ln(3c_{\sigma})\\
\le & -d(t-1)\ln(2)+(d-1)\ln(2)+\frac{1}{2}\ln(c_{p})+d\ln(3c_{\sigma}).
\end{align*}
Thus, for 
\[
t\ge\frac{(2d-1)\ln(2)+d\ln(3c_{\sigma})+\frac{1}{2}\ln(c_{p})}{d}+1\eqqcolon c_{1}
\]
it follows that 
\[
2^{d-1}3(td)^{-3td}c_{p}^{\nicefrac{1}{2}}(3c_{\sigma}td)^{d}\le3(td)^{-2td}.
\]
Combining the estimates of the terms in \eqref{eq:4}, we have for
$t\ge\max\left\{ c_{1},c_{2}\right\} $ that
\[
\left\langle L_{k}W_{m}^{x}-k_{x},W_{m}^{x}\right\rangle _{2}\le6(td)^{-2td}
\]
 and thus the assertion.
\end{proof}

\section{\label{sec:Eigenvalue-Approximation} Decay of eigenvalues}

This section provides a lower bound on the eigenvalues $(\mu_{\ell})_{\ell=1}^{\infty}$
of the operator $L_{k}$ associated with the Gaussian kernel. We address
the univariate case first, which is then extended to the multivariate
case. As a starting point we consider the most elementary setting,
i.e., $\mathcal{X}=[0,1]$, equipped with the uniform measure $P=\mathcal{U}[0,1]$.
The following lemma provides the precise eigenvalue bound. 
\begin{lem}[Maximal decay of eigenvalues]
\label{lem:Eigdec-1-1} Let $k$ be the Gaussian kernel, $\mathcal{X}=[0,1]$
as well as $P=\mathcal{U}[0,1]$. For every $\ell\in\mathbb{N}$ it
holds that
\begin{equation}
\mu_{\ell}\ge\frac{1}{\ell}C(\sigma)e^{-a_{\sigma}(\ell-1)^{2}},\label{eq:Eigdecay-1-1}
\end{equation}
where $a_{\sigma}=8\cdot\frac{4\pi^{2}}{16\sigma}$ and $C(\sigma)$
is a constant depending on the width parameter~$\sigma$.
\end{lem}

\begin{proof}
Let $x_{1},\dots,x_{\ell}$ be independent random variables following
the uniform measure~$\mathcal{U}[0,1]$. Let $K=(k(x_{i},x_{j}))_{i,j=1}^{\ell}$
be the Gramian matrix and invoking \citet[Proposition~A]{Shawe-Taylor2002},
it follows that
\[
\mu_{\ell}\ge\frac{1}{\ell}\E\lambda_{\min}(K),
\]
where $\lambda_{\min}(K)$ is the smallest eigenvalue of the matrix~$K$
and the expectation is with respect to the samples. Further, employing
the result of \citet[Example~2.6]{Diederichs2019}, we get with $M\coloneqq\min_{i,j\le\ell,i\neq j}\left|x_{i}-x_{j}\right|$
and~\eqref{eq:ExpectedDec} in the auxiliary Lemma~\ref{lem:ExpectedEigenvalue}
below (Appendix~\ref{sec:AuxiliaryLemma}) that
\begin{align*}
\E\lambda_{\min}(K) & \ge\E M^{-1}\tilde{C}(\sigma)e^{-\frac{4\pi^{2}}{16M^{2}\sigma}}\ge\tilde{C}(\sigma)\E e^{-\frac{4\pi^{2}}{16M^{2}\sigma}}\ge C(\sigma)e^{-a_{\sigma}(\ell-1)^{2}},
\end{align*}
as $M<1$, and where $C(\sigma)=C\cdot\tilde{C}(\sigma)$ with $C$
from as in~\eqref{eq:ExpectedDec}. This is the assertion.
\end{proof}
Extending the univariate case to the multivariate setting builds on
the product structure of the Gaussian kernel, that is, on $\prod_{i=1}^{d}e^{-\sigma(x_{i}-y_{i})^{2}}=e^{-\sigma\sum_{i=1}^{n}(x_{i}-y_{i})^{2}}$.
Indeed, provided that the underlying measure is $\mathcal{U}[0,1]^{d}$,
the spectrum of the corresponding operator $L_{k}$ is 
\[
\left\{ \prod_{i=1}^{d}\mu_{\ell_{i}}\colon\ell_{1},\dots,\ell_{d}\in\mathbb{N}\right\} ,
\]
where $\mu_{\ell_{i}}$ is the $\ell_{i}$th eigenvalue in the univariate
setting. This is immediate, as every eigenfunction in the multivariate
case is a product of elementary eigenfunctions, $\prod_{i=1}^{d}\varphi_{\ell_{i}}(x_{i})$.
We may assume the multivariate eigenvalues $\mu_{\ell}^{(d)}$ arranged
in non-increasing order such that 
\begin{equation}
\left(\mu_{\ell}^{(d)}\right)_{\ell=1}^{\infty}=\left(\prod_{i=1}^{d}\mu_{\ell_{i}}\colon\ell_{i}\in\mathbb{N}\right)\label{eq:d-dim-eig}
\end{equation}
and $\mu_{1}^{(d)}\ge\mu_{2}^{(d)}\ge\dots$.

The subsequent auxiliary and combinatorial lemmata utilize the structure
of the spectrum to infer the maximal decay rate of the sequence $\mu_{\ell}^{(d)}$.
The first is a general combinatorial result, with which we assess the
eigenvalue decay in the second.
\begin{lem}
It holds that
\begin{equation}
\sum_{\substack{i_{1}+\dots+i_{d}\le n,\\
i_{j}\in\mathbb{N}
}
}1\ge\frac{(n-1)^{d}}{(d-1)^{d-1}}-d\label{eq:Combinatorial}
\end{equation}
for $n\ge d\ge2.$
\end{lem}

\begin{proof}
We proof the assertion by employing on the \emph{stars and bars formula}
\[
\sum_{i_{1}+\dots+i_{d}=i}1=\binom{i-1}{d-1},
\]
where $i_{j}\ge1$ are positive integers (cf.\ \citealt[p. 38]{feller1})
and $i\in\mathbb{N}$. We have that
\begin{align}
\sum_{i_{1}+\dots+i_{d}\le n}1 & =\sum_{i=d}^{n}\sum_{i_{1}+\dots+i_{d}=i}1=\sum_{i=d}^{n}\binom{i-1}{d-1}\ge\sum_{i=d}^{n}\left(\frac{i-1}{d-1}\right)^{d-1}\ge\sum_{i=1}^{n}\left(\frac{i-1}{d-1}\right)^{d-1}-d,\label{eq:5}
\end{align}
where we utilize that 
\[
\binom{i-1}{d-1}=\frac{(i-1)(i-2)\cdots(i-d+1)}{(d-1)(d-2)\cdots1}=\prod_{j=1}^{d-1}\frac{i-j}{d-j}\ge\prod_{j=1}^{d-1}\frac{i-1}{d-1}=\left(\frac{i-1}{d-1}\right)^{d-1}.
\]
Furthermore, employing 
\[
\frac{n^{d}}{d}=\int_{0}^{n}x^{d-1}dx\le\sum_{i=1}^{n}k^{d-1}
\]
in \eqref{eq:5}, it follows that
\begin{align*}
\sum_{i=1}^{n}\left(\frac{i-1}{d-1}\right)^{d-1}-d & =\frac{1}{(d-1)^{d-1}}\sum_{i=1}^{n}\left(i-1\right)^{d-1}-d\\
 & =\frac{1}{(d-1)^{d-1}}\sum_{k=0}^{n-1}k^{d-1}-d\ge\frac{(n-1)^{d}}{d(d-1)^{d-1}}-d
\end{align*}
and thus the assertion.
\end{proof}
\begin{lem}
\label{lem:MultiEigs} Let $\mu_{\ell}\ge e^{-\rho\ell^{2}}$ for
every $\ell\in\mathbb{N}$. Then the sequence $\mu_{\ell}^{(d)}$
from~\eqref{eq:d-dim-eig} satisfies 
\begin{equation}
\mu_{\ell}^{(d)}\ge C\exp(-c\rho(\ell+d)^{\frac{2}{d}})\label{eq:eigdecddim}
\end{equation}
for all $\ell\in\mathbb{N}$ and $d\ge2$ for some $c>0$ and $C>0$.
\end{lem}

\begin{proof}
We first determine the number of combinations of $\mu_{i_{1}},\dots,\mu_{i_{d}}$
for which the product $\prod_{k=1}^{d}\mu_{i_{k}}$ is larger than
$\exp(-\rho(\lambda+2)^{2})$ for some threshold parameter $\lambda>0$.
It holds that $\prod_{i=1}^{d}\mu_{\ell_{i}}\ge e^{-\rho\,i_{1}^{2}-\ldots-\rho\,i_{d}^{2}}\ge e^{-\rho(i_{1}+\dots+i_{d})^{2}}\ge e^{-\rho(\lambda+2)^{2}}$,
provided that $i_{1}+\dots+i_{d}\le\lambda+2$. From~\eqref{eq:Combinatorial}
we deduce that
\[
\sum_{i_{1}+\dots+i_{d}\le\lambda+2}1\ge\frac{\left\lfloor \lambda+1\right\rfloor ^{d}}{d(d-1)^{(d-1)}}-d
\]
and thus
\[
\mu_{\ell}^{(d)}\ge e^{-\rho(\lambda+2)^{2}}
\]
for all $\ell\le\frac{\left\lfloor \lambda+1\right\rfloor ^{d}}{d(d-1)^{(d-1)}}-d.$
Choosing $\lambda\coloneqq\bigl(d(d-1)^{d-1}(\ell+d)\bigr)^{\nicefrac{1}{d}}$
and employing $(a+b)^{2}\le2a^{2}+2b^{2}$, we get that
\begin{align*}
\mu_{\ell}^{(d)} & \ge e^{-\rho\left((d(d-1)^{d-1}(\ell+d)\bigr)^{\nicefrac{1}{d}}+2\right)^{2}}\ge e^{-\rho2^{3}}e^{-\rho2d^{\nicefrac{2}{d}}(d-1)^{\frac{2(d-1)}{d}}(\ell+d)^{\nicefrac{2}{d}}}.
\end{align*}
The assertion follows by setting $C\coloneqq\exp(-8\rho)$ and $c\coloneqq2d^{\nicefrac{2}{d}}(d-1)^{\frac{2(d-1)}{d}}$.

The bound \eqref{eq:eigdecddim} already implies \eqref{eq:Eigdecay-2-1},
provided that the underlying measure is uniform, i.e, the Lebesgue
measure. The following lemma extends the assertion for more general
probability measures.
\end{proof}
\begin{lem}
Let $\mathcal{X}=[0,1]^{d}$ and consider the operators $L_{k}\colon L^{2}(\mathcal{X},\lambda)\to L^{2}(\mathcal{X},\lambda)$
and $L_{k}^{p}\colon L^{2}(\mathcal{X},P)\to L^{2}(\mathcal{X},P)$
with
\[
(L_{k}f)(y)=\int_{\mathcal{X}}k(x,y)f(x)dx\hspace{1em}(L_{k}^{P}f)(y)=\int_{\mathcal{X}}k(x,y)f(x)p(x)dx.
\]
Here, $\lambda$ is the Lebesgue (uniform) measure and $P$ is a probability
measure, satisfying the condition $0<p_{\min}\coloneqq\inf_{x\in\mathcal{X}}p(x)\le p_{\max}\coloneqq\sup_{x\in\mathcal{X}}p(x)<\infty$.
The eigenvalues $(\mu_{\ell}^{(d)})_{\ell}$ of $L_{k}$ ($(\mu_{\ell,p}^{(d)})_{\ell}$
of $L_{k}^{p}$, resp.), satisfy the relation 
\begin{equation}
\mu_{\ell,p}^{(d)}\ge\frac{p_{\min}^{2}}{p_{\max}^{2}}\mu_{\ell}^{(d)}\label{eq:Eigprobest}
\end{equation}
for all $\ell\in\mathbb{N}.$ 
\end{lem}

\begin{proof}
By the Courant\textendash Fischer\textendash Weyl min-max principle
it holds that 
\begin{align*}
\mu_{\ell,p}^{(d)} & =\max_{\dim(S_{k})=k}\min_{\substack{\begin{subarray}{c}
x\in S_{k}\\
\left\Vert x\right\Vert _{L^{2}(P)}=1
\end{subarray}}
}\left\langle L_{k}^{p}x,x\right\rangle _{L^{2}(P)}\\
 & =\max_{\dim(S_{k})=k}\min_{\substack{x\in S_{k}}
}\left(\frac{\|x\|_{\lambda}}{\|x\|_{p}}\right)^{2}\left\langle L_{k}^{p}\frac{x}{\|x\|_{\lambda}},\frac{x}{\|x\|_{\lambda}}\right\rangle _{L^{2}(P)}\\
 & \ge\max_{\dim(S_{k})=k}\min_{\substack{\begin{subarray}{c}
x\in S_{k}\\
\left\Vert x\right\Vert _{L^{2}(\lambda)}=1
\end{subarray}}
}p_{\max}^{-2}\left\langle L_{k}^{p}x,x\right\rangle _{L^{2}(P)}.
\end{align*}
It follows further that 
\begin{align*}
\mu_{\ell,p}^{(d)} & \ge\max_{\dim(S_{k})=k}\min_{\substack{\begin{subarray}{c}
x\in S_{k}\\
\left\Vert x\right\Vert _{L^{2}(\lambda)}=1
\end{subarray}}
}p_{\max}^{-2}\left\langle L_{k}^{p}x,x\right\rangle _{L^{2}(P)}\\
 & \ge\max_{\dim(S_{k})=k}\min_{\substack{\begin{subarray}{c}
x\in S_{k}\\
\left\Vert x\right\Vert _{L^{2}(\lambda)}=1
\end{subarray}}
}p_{\max}^{-2}p_{\min}^{2}\left\langle L_{k}x,x\right\rangle _{L^{2}(\lambda)}=\frac{p_{\min}^{2}}{p_{\max}^{2}}\mu_{\ell}^{(d)},
\end{align*}
as $0\le p(x)p(y)-p_{\min}^{2}$. Hence, the assertion.
\end{proof}
We now combine the results of the preceding lemmata to bound the eigenvalues
of the $d$-dimensional Gaussian kernel for general measures.
\begin{lem}[Maximal decay of eigenvalues]
\label{lem:Eigdec-2} For every $\ell\in\mathbb{N}$ it holds that
\begin{equation}
\mu_{\ell}^{(d)}\ge\frac{p_{\min}^{2}}{p_{\max}^{2}}C(d,\sigma)e^{-c_{\sigma,d}(\ell+d)^{\frac{2}{d}}},\label{eq:Eigdecay-2}
\end{equation}
where $c_{d,\sigma}$ and $C(d,\sigma)$ are constants depending on
the dimension~$d$ and the bandwidth~$\sigma$.
\end{lem}

\begin{proof}
We show the assertion only for the uniform measure $\mathcal{U}[0,1]^{d}$,
as the result for more general design measures follows immediately
by~\eqref{eq:Eigprobest}. For the eigenvalues $\mu_{\ell}$ in the
univariate setting, we have from~\eqref{eq:Eigdecay-1-1} that 
\begin{equation}
\mu_{\ell}\ge\frac{1}{\ell}C(\sigma)e^{-a_{\sigma}(\ell-1)^{2}}\ge C(\sigma)e^{-\tilde{a}_{\sigma}\ell^{2}}\label{eq:dectilde}
\end{equation}
where $\tilde{a}_{\sigma}$ is chosen such that 
\[
\tilde{a}_{\sigma}\ell^{2}\ge\ln(\ell)+a_{\sigma}(\ell-1)^{2},\qquad\ell=1,2,\dots.
\]
Combining \eqref{eq:dectilde} with \eqref{eq:eigdecddim} (cf.\ Lemma~\ref{lem:MultiEigs}),
it follows that
\[
\mu_{\ell}^{(d)}\ge C(\sigma)^{d}C\exp(-\tilde{a}_{\sigma}c(\ell+d)^{\frac{2}{d}})
\]
as $C(\sigma)$ appears in every factor of the product $\prod_{i=1}^{d}\mu_{\ell_{i}}$.
Setting $c_{\sigma,d}\coloneqq\tilde{a}c$ reveals the assertion. 
\end{proof}

\section{\label{sec:ConcentrationAPP} Concentration}

This section provides a proof of the operator bound~\eqref{eq:ConcInequExakt-1}.
To this end, we restate the following concentration bound for random
operators on Hilbert spaces, which we then subsequently.
\begin{prop}[{See \citealt[Theorem A.3]{steinwartfischer2020}}]
\label{prop:ConcResultRight} Let $(\Omega,\mathcal{B},P)$ be a
probability space, $H$ a separable Hilbert space, and $\xi\colon\Omega\to\mathcal{L}_{2}(H)$
be a random variable with values in the set of self-adjoint Hilbert-Schmidt
operators. Furthermore, let the operator norm be uniformly, i.e.,
$\left\Vert \xi\right\Vert _{H\to H}\le B$ $P$-a.s.\ and $V$ be
a self-adjoint positive semi-definite trace class operator with $\E_{P}\xi^{2}\preccurlyeq V$,
i.e.\ $V-\E_{P}\xi^{2}$ is positive semi-definite. Then, for $g(V)\coloneqq\ln(2e\Tr(V)\left\Vert V\right\Vert _{H\to H}^{-1})$,
$\tau\ge1$ and $n\ge1$, the following concentration inequality is
satisfied: 
\begin{equation}
P^{n}\bigg(\left\Vert \frac{1}{n}\sum_{i=1}^{n}\xi_{i}-\E_{P}\xi\right\Vert _{H\to H}\ge\frac{4\tau\,B\cdot g(V)}{3n}+\sqrt{\frac{2\tau\left\Vert V\right\Vert _{H\to H}g(V)}{n}}\bigg)\le2e^{-\tau}.\label{eq:ConcGeneral}
\end{equation}
\end{prop}

To demonstrate the desired bound~\eqref{eq:ConcInequExakt-1}, we
show first that $\left(L_{k}+\lambda\right)^{-1}L_{k}$ and $\left(L_{k}+\lambda\right)^{-1}L_{k}^{D}$
are close in operator norm. To this end we rephrase the operator $L_{k}^{D}$
in terms of simple operators. Letting $x_{1},\dots,x_{n}\sim P$,
independently distributed with respect to the design measure, and
defining the operators 
\begin{equation}
T_{z}\colon\mathcal{H}_{k}\to\mathcal{H}_{k}\hspace{1em}\text{with}\hspace{1em}(T_{z}f)(y)=f(z)k(y,z)\label{eq:PointOperator}
\end{equation}
gives the representation
\[
L_{k}^{D}=\frac{1}{n}\sum_{i=1}^{n}T_{x_{i}},
\]
which fits the setting of Proposition~\ref{prop:ConcResultRight}.
With this we have the subsequent concentration bound, following from the same arguments as Proposition~3.4 in \citet{dommel2024bound}. 
\begin{prop}
For $\mathcal{N}_{\infty}(\lambda)$ as in~\eqref{eq:Ninf} it holds
that
\begin{equation}
P\left(\left\Vert \left(L_{k}+\lambda\right)^{-1}(L_{k}-L_{k}^{D})\right\Vert _{\mathcal{H}_{k}\to\mathcal{H}_{k}}\le\frac{4\tau\mathcal{N}_{\infty}(\lambda)g(\lambda)}{3n}+\sqrt{\frac{2\tau\mathcal{N}_{\infty}(\lambda)g(\lambda)}{n}}\right)\ge1-2e^{-\tau}\label{eq:ConcInequExakt}
\end{equation}
with 
\begin{equation}
g(\lambda)\coloneqq\ln\left(2e\frac{\lambda+\mu_{1}}{\mu_{1}}\mathcal{N}(\lambda)\right)\label{eq:glambda}
\end{equation}
and $\tau\ge1$. 
\end{prop}

\begin{proof}
For $x\sim P$ and $T_{x}$ as in~\eqref{eq:PointOperator} we consider
the operator-valued random variable $\xi(\omega):=(L_{k}+\lambda)^{-1}T_{x(\omega)}$.
We have from 
\[
\E(T_{x}f)(y)=\E_{x}f(x)k(y,x)=\int_{\mathcal{X}}f(x)k(y,x)p(x)dx=(L_{k}f)(y)
\]
that
\[
\frac{1}{n}\sum_{i=1}^{n}\xi_{i}=(L_{k}+\lambda)^{-1}L_{k}^{D}\hspace{1em}\text{and}\hspace{1em}\E\xi=(L_{k}+\lambda)^{-1}L_{k},
\]
and thus the setting of Proposition~\ref{prop:ConcResultRight}.
It is thus sufficient to show that $\xi$ satisfies the requirements
of Proposition~\ref{prop:ConcResultRight}, provided that there is
an appropriate constant $B$ as well as dominating operator $V$. 

We bound the norm $\left\Vert \xi\right\Vert _{\mathcal{H}_{k}\to\mathcal{H}_{k}}$
first. For that recall the inequality
\[
\left\Vert \xi\right\Vert _{\mathcal{H}_{k}\to\mathcal{H}_{k}}^{2}\le\sum_{\ell=1}^{\infty}\left\langle \psi_{\ell},\xi\psi_{\ell}\right\rangle _{k}^{2},
\]
which holds for any orthonormal basis $\left(\psi_{\ell}\right)_{\ell=1}^{\infty}$of
$\mathcal{H}_{k}$. Thus, setting $\psi_{\ell}:=\mu_{\ell}^{\nicefrac{1}{2}}\varphi_{\ell}$
we observe that
\begin{align*}
\left\Vert \xi\right\Vert _{\mathcal{H}_{k}\to\mathcal{H}_{k}}^{2}\le & \sum_{\ell=1}^{\infty}\left\langle \psi_{\ell},(L_{k}+\lambda)^{-1}T_{x}\psi_{\ell}\right\rangle _{k}^{2}=\sum_{\ell=1}^{\infty}\left\langle (L_{k}+\lambda)^{-1}\psi_{\ell}(\cdot),\psi_{\ell}(x)k(\cdot,x)\right\rangle _{k}^{2}
\end{align*}
as $(L_{k}+\lambda)^{-1}$ is a self-adjoint operator. Furthermore,
using the reproducing property, i.e.\ the identity $\left\langle (L_{k}+\lambda)^{-1}\psi_{\ell}(\cdot),k(\cdot,x)\right\rangle _{k}=\left((L_{k}+\lambda)^{-1}\psi_{\ell}\right)(x)$,
we get for the latter term that
\begin{align*}
\sum_{\ell=1}^{\infty}\left\langle (L_{k}+\lambda)^{-1}\psi_{\ell}(\cdot),\psi_{\ell}(x)k(\cdot,x)\right\rangle _{k}^{2} & =\sum_{\ell=1}^{\infty}\left(\psi_{\ell}(x)((L_{k}+\lambda)^{-1}\psi_{\ell})(x)\right)^{2},
\end{align*}
from which we conclude
\begin{align}
\left\Vert \xi\right\Vert _{\mathcal{H}_{k}\to\mathcal{H}_{k}}^{2} & \le\sum_{\ell=1}^{\infty}\left(\psi_{\ell}(x)((L_{k}+\lambda)^{-1}\psi_{\ell})(x)\right)^{2}=\sum_{\ell=1}^{\infty}\left(\frac{\mu_{\ell}}{\mu_{\ell}+\lambda}\varphi_{\ell}(x)^{2}\right)^{2}\le\mathcal{N}_{\infty}(\lambda)^{2}\label{eq:EstTrace}
\end{align}
$P$-almost surely, which discloses the constant $B=\mathcal{N}_{\infty}(\lambda)$.

To bound the second moment, note first that $\xi$ is a positive definite
operator. Hence, we have that
\[
\E\xi^{2}\preccurlyeq\left\Vert \xi\right\Vert _{\mathcal{H}_{k}\to\mathcal{H}_{k}}\E\xi=\left\Vert \xi\right\Vert _{\mathcal{H}_{k}\to\mathcal{H}_{k}}(L_{k}+\lambda)^{-1}L_{k}\le\mathcal{N}_{\infty}(\lambda)(L_{k}+\lambda)^{-1}L_{k}\eqqcolon V
\]
by employing the estimate in \eqref{eq:EstTrace}. The operator norm
of $V$ is bounded by
\[
\left\Vert V\right\Vert _{\mathcal{H}_{k}\to\mathcal{H}_{k}}=\left\Vert \mathcal{N}_{\infty}(\lambda)(L_{k}+\lambda)^{-1}L_{k}\right\Vert _{\mathcal{H}_{k}\to\mathcal{H}_{k}}=\mathcal{N}_{\infty}(\lambda)\frac{\mu_{1}}{\lambda+\mu_{1}}\le\mathcal{N}_{\infty}(\lambda)
\]
as well as
\begin{align*}
g(V) & =\ln\bigl(2e\Tr(V)\cdot\left\Vert V\right\Vert _{\mathcal{H}_{k}\to\mathcal{H}_{k}}^{-1}\bigr)=\ln\left(2e\mathcal{N}_{\infty}(\lambda)\mathcal{N}(\lambda)\cdot\frac{1}{\mathcal{N}_{\infty}(\lambda)\frac{\mu_{1}}{\lambda+\mu_{1}}}\right)=\ln\left(2e\frac{\lambda+\mu_{1}}{\mu_{1}}\mathcal{N}(\lambda)\right)
\end{align*}
corresponding to $g(\lambda)$ in~\eqref{eq:glambda}. The desired
inequality~\eqref{eq:ConcInequExakt} follows from Proposition~\ref{prop:ConcResultRight}. 
\end{proof}
Building on the bound \eqref{eq:ConcInequExakt} above, the assertion
of Proposition~\ref{prop:Concentration} follows from the subsequent
considerations. Note first the operator identity
\[
(I-(L_{k}+\lambda)^{-1}(L_{k}-L_{k}^{D}))^{-1}=(L_{k}^{D}+\lambda)^{-1}(L_{k}+\lambda),
\]
from which we conclude that
\begin{align*}
\left\Vert (L_{k}^{D}+\lambda)^{-1}(L_{k}+\lambda)\right\Vert _{\mathcal{H}_{k}\to\mathcal{H}_{k}} & =\left\Vert (I-\left(L_{k}+\lambda\right)^{-1}(L_{k}-L_{k}^{D}))^{-1}\right\Vert _{\mathcal{H}_{k}\to\mathcal{H}_{k}}\\
 & \le\frac{1}{1-\left\Vert \left(L_{k}+\lambda\right)^{-1}(L_{k}-L_{k}^{D})\right\Vert _{\mathcal{H}_{k}\to\mathcal{H}_{k}}}
\end{align*}
by involving the Neumann series. Assuming the condition \eqref{eq:LambdaCond},
we have by \eqref{eq:ConcInequExakt} that
\[
\left\Vert \left(L_{k}+\lambda\right)^{-1}(L_{k}-L_{k}^{D})\right\Vert _{\mathcal{H}_{k}\to\mathcal{H}_{k}}\le\frac{4}{3}\tau\,g(\lambda)\frac{\mathcal{N}_{\infty}(\lambda)}{n}+\sqrt{2\tau\,g(\lambda)\frac{\mathcal{N}_{\infty}(\lambda)}{n}}\le\frac{1}{2}
\]
with probability at least $1-2e^{-\tau}$. Therefore, the bound
\[
\left\Vert (L_{k}^{D}+\lambda)^{-1}(L_{k}+\lambda)\right\Vert _{\mathcal{H}_{k}\to\mathcal{H}_{k}}\le\frac{1}{1-\left\Vert \left(L_{k}+\lambda\right)^{-1}(L_{k}-L_{k}^{D})\right\Vert _{\mathcal{H}_{k}\to\mathcal{H}_{k}}}\le\frac{1}{1-\frac{1}{2}}=2
\]
 holds also with probability at least $1-2e^{-\tau}$. This is the
desired inequality \eqref{eq:ConcInequExakt-1}.

\section{\label{sec:AuxiliaryLemma}Auxiliary lemmata}

The following two lemmata provide a crucial element for the proof
of Lemma~\ref{lem:Eigdec-1-1}. That is, a bound on the expectation
\[
\E e^{-\frac{1}{M^{2}}},
\]
where 
\[
M\coloneqq\min_{\substack{i,j=1,\dots,n\\
i\not=j
}
}|U_{i}-U_{j}|
\]
is the minimal gap between $n$ independently chosen uniforms.

The first lemma provides the density of $M$ explicitly, the other
bounds the associated expectation.
\begin{lem}
\label{lem:LemmaMinUi} Let $U_{1},\dots,U_{n}\sim\mathcal{U}[0,1]$
be independent uniforms. The random variable $M$ has the density
\begin{equation}
p_{M}(m)=\begin{cases}
n(n-1)(1-(n-1)m)^{n-1} & \text{for}\,m\in\left[0,\frac{1}{n-1}\right]\\
0 & \text{else}
\end{cases}.\label{eq:density}
\end{equation}
 
\end{lem}

\begin{proof}
Let $U_{1},\dots,U_{n}$ be independent uniforms on $[0,1]$ and denote
the corresponding minimal absolute difference by $M\coloneqq\min_{i,j=1,\dots,n,i\neq j}|U_{i}-U_{j}|$.
Note here that $M\le\frac{1}{n-1}$, and $M=\frac{1}{n-1}$ if all
$U_{1},\dots,U_{n}$ are equidistant. Therefore, let $m\in[0,\frac{1}{n-1}]$
and observe that 
\begin{equation}
P(M>m)=n!P(M>m,U_{1}\le U_{2}\dots\le U_{n})\label{eq:Rearrange}
\end{equation}
as there are $n!$ possible rearrangements of the random variables
$U_{1},\dots,U_{n}$. For the latter probability we have that 
\begin{align*}
P(M>m,\ U_{1}\le U_{2}\le\ldots\le U_{n}) & =P(U_{1}\le U_{2}-m,\ U_{2}\le U_{3}-m,\dots,U_{n-1}\le U_{n}-m,\ U_{1}\le\ldots\le U_{n})\\
 & =\lambda\left(\mathcal{U}_{m}\right),
\end{align*}
where $\lambda\left(\mathcal{U}_{m}\right)$ is the Lebesgue measure
of the set
\[
\mathcal{U}_{m}=\left\{ (u_{1},\dots,u_{n})\in[0,1]^{n}:u_{1}\le u_{2}-m,\ u_{2}\le u_{3}-m,\dots,u_{n-1}\le u_{n}-m,\ u_{1}\le u_{2}\dots\le u_{n}\right\} .
\]
We next present a measure persevering bijection between $\mathcal{U}_{m}$
and 
\[
Y_{m}\coloneqq\left\{ (y_{1},\dots,y_{n})\in[0,1-(n-1)m]^{n}:y_{1}\le y_{2}\le\dots\le y_{n}\right\} .
\]
To this end, define $T\colon U_{m}\to Y_{m}$ with $Tu=u-(0,m,2m,\dots,(n-1)m)$.
For $u=(u_{1},\dots,u_{n})\in U_{m}$ and $y=Tu=u-(0,m,2m,\dots,(n-1)m)$
it is evident that $y_{i}\ge0$ as well as $y_{1}\le\dots\le y_{n}$.
Furthermore, it holds that $u_{i}\le1-m\,(n-i)$ as the distance between
the $u_{i}$ and $u_{i+1}$ is at least $m$, for every $i=1,\dots,n$.
With this we have the inequality 
\[
y_{i}=u_{i}-(i-1)m\le1-m\,(n-i)-(i-1)m=1-(n-1)m
\]
and therefore $y\in Y_{m}$. Conversely, let $y\in Y_{m}$ and set
$u=y+(0,m,2m,\dots,(n-1)m)=T^{-1}y$. It is again immediate that $u_{i}\ge0$
as well as 
\[
u_{i}=y_{i}+(i-1)m\le1-(n-1)m+(i-1)m\le1.
\]
 From $y_{1}\le\dots\le y_{n}$ we further get that
\[
u_{i+1}-m=y_{i+1}+i\,m-m=y_{i+1}+(i-1)m\ge y_{i}+(i-1)m=u_{i}
\]
for all $i=1,\dots,n-1$ and thus $u\in\mathcal{U}_{m}$. Hence, $T$
is a bijection, from which we conclude that $\lambda(U_{m})=\lambda(Y_{m})$.
The latter measure is
\begin{align}
\lambda(Y_{m}) & =\lambda\left(\left\{ (y_{1},\dots,y_{n})\in[0,1-(n-1)m]^{n}:y_{1}\le y_{2}\le\dots\le y_{n}\right\} \right)\nonumber \\
 & =\frac{1}{n!}\lambda\left(\left\{ (y_{1},\dots,y_{n})\in[0,1-(n-1)m]^{n}\right\} \right)=\frac{1}{n!}(1-(n-1)m)^{n}.\label{eq:MeasureY}
\end{align}
Combining \eqref{eq:Rearrange} with \eqref{eq:MeasureY}, we get
that 
\[
P(M>m)=n!\frac{1}{n!}(1-(n-1)m)^{n}=(1-(n-1)m)^{n}
\]
and hence the density
\[
p(m)=\frac{d}{dm}(1-P(M>m))=n(n-1)(1-(n-1)m)^{n-1}.
\]
This is the assertion.
\end{proof}
\begin{lem}
\label{lem:ExpectedEigenvalue} Let $U_{1},\dots,U_{n}\sim\mathcal{U}[0,1]$
be independent uniforms and $M$ the minimum gap as in Lemma~\ref{lem:LemmaMinUi}.
For any $c>0$ it holds that
\begin{equation}
\E e^{-cM^{-2}}\ge4e^{-c\frac{2}{a^{2}}}\ge Ce^{-8c(n-1)^{2}},\label{eq:ExpectedDec}
\end{equation}
with $a=\min\{\frac{1}{3}c^{-\frac{2}{3}},\ \frac{1}{2(n-1)}\}$.
\end{lem}

\begin{proof}
We employ the density \eqref{eq:density} of $M$. As $\frac{1}{n-1}\le\frac{1}{2(n-1)}$
and $e^{-cM^{-2}}\ge0$, we have that
\begin{align*}
\E e^{-cM^{-2}} & =\int_{0}^{\frac{1}{n-1}}e^{-c\frac{1}{m^{2}}}\,n(n-1)\bigl(1-(n-1)m\bigr)^{n-1}dm\\
 & >\int_{0}^{\frac{1}{2(n-1)}}e^{-c\frac{1}{m^{2}}}n(n-1)(1-(n-1)m)^{n-1}dm\\
 & \ge e^{-(n-1)}\int_{0}^{\frac{1}{2(n-1)}}e^{-c\frac{1}{m^{2}}}dm.
\end{align*}
To bound the latter integral term, note that 
\begin{align*}
\frac{x^{3}}{c}e^{-\frac{x^{2}}{c}} & \le1
\end{align*}
whenever $x\ge3c^{\frac{2}{3}}$, as 
\[
\ln\frac{x^{3}}{c}e^{-\frac{x^{2}}{c}}=3\ln\frac{x}{c^{\frac{1}{3}}}-\frac{x^{2}}{c}\le3\frac{x}{c^{\frac{1}{3}}}-\frac{x^{2}}{c}\le0
\]
is satisfied for all $x\ge3c^{\frac{2}{3}}.$ Thus, choosing $a=\min\{\frac{1}{3}c^{-\frac{2}{3}},\frac{1}{2(n-1)}\}$
we get that 
\begin{align*}
\int_{0}^{\frac{1}{2(n-1)}}e^{-c\frac{1}{x^{2}}}dx & \ge\int_{0}^{a}\frac{c}{x^{3}}e^{-c\frac{2}{x^{2}}}dx=\left[4e^{-c\frac{2}{x^{2}}}\right]_{0}^{a}=4e^{-c\frac{2}{a^{2}}},
\end{align*}
which is the assertion.
\end{proof}

\end{document}